\documentclass[11pt]{article}
\usepackage[margin=1in]{geometry}
\usepackage{amsmath,amssymb,amsthm}
\usepackage{graphicx}
\usepackage{booktabs}
\usepackage{multirow}
\usepackage[numbers,sort&compress]{natbib}
\usepackage[colorlinks=true,linkcolor=blue,citecolor=blue,urlcolor=blue]{hyperref}
\usepackage{microtype}

\graphicspath{{../figures/}{figures/}}

\newtheorem{theorem}{Theorem}
\newtheorem{proposition}{Proposition}
\newtheorem{lemma}{Lemma}
\newtheorem{conjecture}{Conjecture}
\newtheorem{assumption}{Assumption}
\newtheorem{remark}{Remark}

\newcommand{\bmu}{\boldsymbol{\mu}}

\title{Neural Collapse Is Forbidden:\\ Information Floors in Language Models}

\author{Bruno Abrahao\\ NYU Shanghai\\ Leonard N. Stern School of Business, New York University\\ \texttt{bd58@nyu.edu}}
\date{July 2026}
\begin{document}
\maketitle
\begin{abstract}
Within-class variance in language-model representations is commonly read as incomplete
neural collapse. We argue it is allocated information storage, and that the allocation
obeys a law. A one-line centering identity voids a family of simplex
equiangular-tight-frame claims, including our own earlier ones; in dimensionless
variance shares across 14 models, macro-category structure carries only 4--12\% of
representational variance and within-token context carries 79--91\%, stable across a
100x parameter range. On the theory side, token-level weight decay penalizes a category
in proportion to its type count, not its occurrence mass, reducing next-token prediction
to an imbalanced $K$-class problem whose optimum orders category norms by type count. A
converse floor, proved for binary categories, forces within-category dispersion to be at
least proportional to the conditional mutual information
$I(\text{token};\text{context}\mid\text{category})$. The law holds: identity dispersion,
not total variance, tracks this information across every tested model and partition,
under a model-free estimate and even across models, where one model's information
predicts another's dispersion; and over pretraining the category share overshoots,
decays, and partially recovers, because the information it must carry never left.
\end{abstract}

\section{Introduction}
\label{sec:v2intro}

In image classification, the terminal phase of training is when geometry tightens:
within-class variability collapses and class means settle into a simplex equiangular
tight frame, the neural collapse (NC) of \citet{papyan2020prevalence}. Language models
invert this script. Measuring the geometry of hidden states over public pretraining
checkpoints, we find that category-level structure crystallizes almost immediately,
overshoots, then partially dissolves over the long remainder of training, and finally,
in most model sizes, partially returns, all while perplexity improves monotonically. The
terminal phase of a language model does not finish collapsing its representations; it
reallocates them. Figure~\ref{fig:v2hook} previews the three findings this inversion
opens.

\begin{figure*}[t]
\centering
\includegraphics[width=\textwidth]{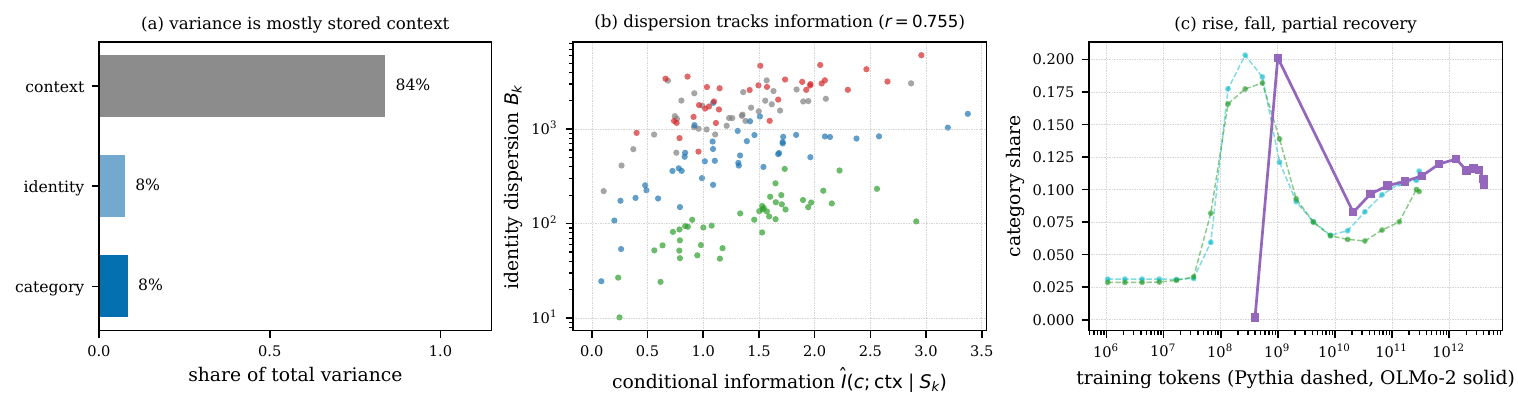}
\caption{\textbf{(a)}~Across 13 models, variance is dominated
by within-token context; macro-category structure is a thin slice, stable across a
100x parameter range. \textbf{(b)}~Within-category identity dispersion tracks
model-realized conditional information $\hat I(c; \mathrm{ctx} \mid S_k)$ (pooled
partial $r = 0.755$), the dispersion channel the information floor of
Section~\ref{sec:v2theory} makes necessary.
\textbf{(c)}~Over pretraining the between-category share overshoots, decays, and
partially recovers, in Pythia (dashed) and OLMo-2 (solid) on a shared token axis.}
\label{fig:v2hook}
\end{figure*}

We argue that the resolution of this apparent paradox is informational rather than
geometric, and that it reframes what within-class variance in language models is.
The prevailing reading, inherited from classification, treats residual within-class
variability as unfinished collapse: noise that more scale or more training should
remove. On measurement grounds alone this reading fails. Using a dimensionless
three-level decomposition of representational variance (within-token context
variability, token identity within category, and category structure), we find across 14
models from three families (GPT-2, Pythia, Qwen2.5; 70M--6.9B parameters) that
within-token context variability carries 79--91\% of total variance (69\% in one
degenerate model), token identity 4--13\%, and macro-category structure only 4--12\%
(28\% in that same degenerate model), and that this allocation is roughly
stable across a 100x range of parameters and across coverage regimes from 500 to 26{,}000
token types. The allocation looks set by the data, not by capacity. Meanwhile roughly
60\% of the models' cross-entropy loss is resolved \emph{within} categories at $K{=}10$,
so the variance-dominant component is also the loss-dominant one.

Our account is that this variance is where the model stores context, and that the
allocation obeys a law with a proved floor. On the theory side
(Section~\ref{sec:v2theory}) we make three moves. First, a one-line centering identity
shows that the mean pairwise cosine of any centered, near-equinorm configuration is
pinned near $-1/(K{-}1)$, which voids a family of simplex-ETF measurements, including
ones we ourselves reported in an earlier version of this work.

Second, an exact
decomposition shows that token-level weight decay does not act uniformly on category
structure. It penalizes a category's mean head row in proportion to the category's
\emph{type count} (the number of distinct tokens it contains), while the category's
occurrence \emph{mass} (its total frequency) enters only through the loss. Under a
context-free within-category readout, next-token prediction then reduces to an
imbalanced $K$-class problem, exactly when the within-category conditionals are uniform
and up to an explicit coupling term in general. Solving that problem at $K{=}2$ shows
its norms are ordered by type count, through the induced decay and offsets, not by
mass, so the optimum is a distorted frame rather than a regular simplex; for general
$K$ we obtain the optimum in closed form (an oblique projection reducing to SELI's
centering matrix at equal decay), with a complete proof, and test it empirically.

Third, a converse floor: if
within-category choice depends on context, features cannot collapse; we prove, in the
binary case, that dispersion along the readout directions is at least proportional to
the conditional mutual information $I(\mathrm{token};\mathrm{context}\mid\mathrm{category})$.

The floor bounds total within-category dispersion, and so establishes that
context-dependent choice requires dispersion without saying which component carries it.
We therefore test the components separately, and find the requirement localizes: the
token-identity dispersion of class means within a category tracks conditional
information (Section~\ref{sec:v2law}), while total variance does not, consistent with a
static within-category preference being expressible by head rows while a
context-dependent one is not.

Estimating the model-realized conditional information per category from a single
inference pass, we find that identity dispersion tracks it in 16 of 16 model-partition
combinations (four models, four partitions), with Spearman correlations of 0.58--0.85
wherever the partition supplies enough categories, and a pooled within-model-rank
partial correlation of $r{=}0.755$ ($p{=}3{\times}10^{-30}$) after controlling for
marginal entropy, category mass, and type count; repeating the estimator over the full
24k-type covered vocabulary on the 50{,}000-sequence train corpus reproduces the law in 12 of 12
further combinations (pooled partial $r{=}0.773$, $n{=}200$ categories). Total variance shows nothing (mixed
signs near zero), and identity dispersion tracks conditional information about twice as
strongly as it tracks marginal entropy: the negative controls behave.

Section~\ref{sec:v2frame}
tests the frame side of the reduction: centered category-centroid norms decrease in 14
of 14 models, in two separable channels; the type-count channel is the one our $K{=}2$
solution predicts at first order, while the mass channel is largely inherited from a
token-level frequency-norm relation and is controlled with frequency-matched nulls.
The same asymmetry holds when the prediction is read directly off the models'
unembedding rows, in 24 of 24 model-regime-partition cells.

Section~\ref{sec:v2dynamics} returns to dynamics with the same dimensionless
instruments. On Pythia checkpoints the between-category share rises abruptly at steps
32--64, peaks at 0.15--0.21 within the first thousand steps, decays to 0.05--0.07 by
mid-training, and then recovers by a factor of 1.5--1.8 in four of six sizes; the
within-token share correspondingly dips and re-expands, a pattern consistent with,
though not by itself demonstrating, a growing storage cost of context-dependent choice.
We repeat the protocol on OLMo-2 pretraining checkpoints under pre-registered criteria:
three of four pass, and the trajectory's minimum lands at the same token count as
Pythia's. Transient semantic geometry in next-token prediction has
recently been reported in synthetic settings as a monotone dissolution into symmetry
\citep{zhao2026structure}, and non-monotonic spectral geometry has been traced on real
checkpoints without category structure or a collapse lens \citep{li2025tracing}; the
decomposition we use shows where the non-monotonicity lives, and that on real models the
story ends differently: category structure partially returns, because the information it
must carry never left.

\paragraph{Contributions.}
\begin{enumerate}
\item \textbf{Measurement.} A centering identity that retires mean-cosine ETF evidence;
a dimensionless three-level variance decomposition; the allocation picture across 14
models and two coverage regimes; and a pitfalls appendix in which covariance-matched
Gaussian nulls cut the k-means-versus-shuffle separations that magnitude claims (ours
included) have rested on down to an excess factor of 1.1--1.75 (a lower bound, the null
being deliberately strict, but a factor rather than orders of magnitude).
\item \textbf{Theory.} An exact reduction of next-token prediction with token-level
weight decay to a size-weighted, offset, imbalanced $K$-class problem
(Theorem~\ref{thm:reduction}); the induced double weighting (mass in the risk, type
count in the decay and offsets), with the $K{=}2$ case solved exactly and showing that
decay, not mass, orders the norms (Proposition~\ref{prop:k2}); and a proved
information floor on within-category dispersion (Proposition~\ref{prop:floor}), with
the general geometry obtained in closed form and fully proved
(Theorem~\ref{conj:seli}).
\item \textbf{An empirical law.} Within-category identity dispersion tracks
model-realized conditional information, in the theoretically discriminating pattern, in
every model and partition tested.
\item \textbf{Dynamics.} On public pretraining checkpoints of two families (Pythia,
and OLMo-2 under pre-registered criteria), category structure crystallizes,
overshoots, decays, and partially recovers, with the post-overshoot minimum aligned
across families on the token axis; the pattern is invisible to mean-cosine metrics,
distinct from global spectral accounts, and contradicts the extrapolation from
synthetic settings that semantic structure monotonically dissolves.
\end{enumerate}

\paragraph{What is proved, what is measured, what is conjectured.}
We tag every claim inline. Every theory claim carries a complete proof; the general-$K$
geometry's final sign condition, numerically certified in an earlier draft, is now
proved in closed form (Appendix~\ref{app:twisted}), and existing imbalanced
unconstrained-features-model (UFM) theorems, which cover neither per-class decay nor
offsets, are not relied on for it. The scope notes worth stating up front: the
general-$K$ floor remains a conjecture (the binary case is proved), and the dynamics
onset at steps 32--64 is resolvable only on
Pythia, whose public checkpoint grid is dense enough; OLMo-2's grid is not. An appendix
audits, by name, the claims of our own earlier version that the present instruments
retire.

\section{Measuring variance allocation without artifacts}
\label{sec:v2measurement}

Before we can ask whether residual within-class variance is unfinished collapse or
stored information, we need an allocation measurement that survives its own artifacts.
This section builds one and reports what it finds; Appendix~\ref{app:v2pitfalls}
records the metrics that did not survive.

\paragraph{Extraction protocol.}
For each model we compute per-token class statistics on WikiText-103
\citep{merity2017pointer}: contexts are packed into 512-token sequences, and for every
vocabulary item $c$ occurring as a next token at least 50 times we record the mean
$\bmu_c$, per-dimension variance, and count $n_c$ of the hidden states at layer $L{-}1$
over positions whose next token is $c$. On the validation split this yields roughly
500--560 covered types per model (the tokenizer-dependent frequent-token subset); on the
train split (50{,}000 sequences) coverage rises to 24{,}000--26{,}000 types, placing all
models squarely in the overcomplete regime ($C/d \in [12, 34]$). We analyze 14
pretrained models across GPT-2 (124M--1.5B), Pythia (70M--6.9B)
\citep{biderman2023pythia}, and Qwen2.5 (0.5B--3B). Categories are formed by $k$-means
on the class means ($K{=}10$, fixed seed) as the illustrative partition; nothing depends
on this choice, and we sweep $K \in \{10, 20, 50\}$ and an exogenous universal
part-of-speech partition \citep{petrov2012universal} throughout. Clusters are seed-stable
(mean adjusted Rand index 0.76 across restarts) and align partially with coarse
part-of-speech structure ($V$-measure 0.37).

\paragraph{The three-level decomposition.}
All headline quantities are dimensionless shares from the exact law of total variance.
With counts $n_c$ as weights, category assignments $k(c)$, mass-weighted category means
$\mathbf{m}_k$ and grand mean $\bmu_G$, total variance splits exactly into
\[
\underbrace{\textstyle\sum_c n_c \sigma_c^2}_{\text{within token}}
\;+\;
\underbrace{\textstyle\sum_c n_c \|\bmu_c - \mathbf{m}_{k(c)}\|^2}_{\text{token identity within category}}
\;+\;
\underbrace{\textstyle\sum_k N_k \|\mathbf{m}_k - \bmu_G\|^2}_{\text{between category}},
\]
where $\sigma_c^2$ is the summed per-dimension within-token variance and $N_k$ the
category mass. Shares are each term over the total. Because shares are unitless they are
comparable across embedding dimensions, model scales, checkpoints, and coverage regimes;
Appendix~\ref{app:v2pitfalls} documents how dimensional metrics (the $\alpha{=}2$
variant of class-distance normalized variance, CDNV) manufactured a spurious
two-orders-of-magnitude scaling trend in an earlier
version of this work.

\begin{table}[t]
\centering
\small
\caption{Variance allocation at $L{-}1$ (count-weighted shares; $K{=}10$ illustrative
partition). CDNV is the standard ($\alpha{=}0$) cluster-level value, shown for
comparability with the NC literature. Bottom block: train-split extractions in the
overcomplete regime. GPT-2 XL is a known degenerate case (near-collinear centroids).}
\label{tab:v2allocation}
\begin{tabular}{lrrrrr}
\toprule
\textbf{Model} & \textbf{Types} & \textbf{Within-token} & \textbf{Identity} & \textbf{Between-cat.} & \textbf{CDNV}$_{\alpha=0}$ \\
\midrule
Pythia-70M   & 559 & 0.825 & 0.088 & 0.087 & 4.62 \\
Pythia-160M  & 559 & 0.797 & 0.090 & 0.113 & 4.57 \\
Pythia-410M  & 559 & 0.840 & 0.064 & 0.096 & 4.58 \\
Pythia-1B    & 559 & 0.914 & 0.039 & 0.047 & 9.41 \\
Pythia-1.4B  & 559 & 0.857 & 0.062 & 0.082 & 5.12 \\
Pythia-2.8B  & 559 & 0.834 & 0.078 & 0.088 & 5.05 \\
Pythia-6.9B  & 559 & 0.885 & 0.051 & 0.064 & 8.05 \\
GPT-2 Small  & 525 & 0.905 & 0.051 & 0.043 & 9.91 \\
GPT-2 Medium & 525 & 0.788 & 0.092 & 0.120 & 3.31 \\
GPT-2 Large  & 525 & 0.808 & 0.098 & 0.094 & 4.68 \\
GPT-2 XL     & 525 & 0.688 & 0.036 & 0.276 & 2.55 \\
Qwen2.5-0.5B & 503 & 0.798 & 0.105 & 0.098 & 4.29 \\
Qwen2.5-1.5B & 503 & 0.802 & 0.099 & 0.099 & 4.15 \\
Qwen2.5-3B   & 503 & 0.847 & 0.079 & 0.074 & 5.08 \\
\midrule
Pythia-70M (train)   & 24\,403 & 0.829 & 0.100 & 0.071 & 6.13 \\
Pythia-410M (train)  & 24\,403 & 0.835 & 0.081 & 0.084 & 6.96 \\
Pythia-1B (train)    & 24\,403 & 0.907 & 0.049 & 0.045 & 12.4 \\
Pythia-1.4B (train)  & 24\,403 & 0.845 & 0.073 & 0.083 & 7.58 \\
Pythia-6.9B (train)  & 24\,403 & 0.874 & 0.062 & 0.064 & 10.2 \\
GPT-2 Small (train)  & 26\,413 & 0.907 & 0.062 & 0.031 & 12.7 \\
GPT-2 XL (train)     & 26\,413 & 0.675 & 0.045 & 0.281 & 4.95 \\
Qwen2.5-1.5B (train) & 25\,266 & 0.800 & 0.125 & 0.075 & 6.77 \\
\bottomrule
\end{tabular}
\end{table}

\begin{figure}[t]
\centering
\includegraphics[width=\textwidth]{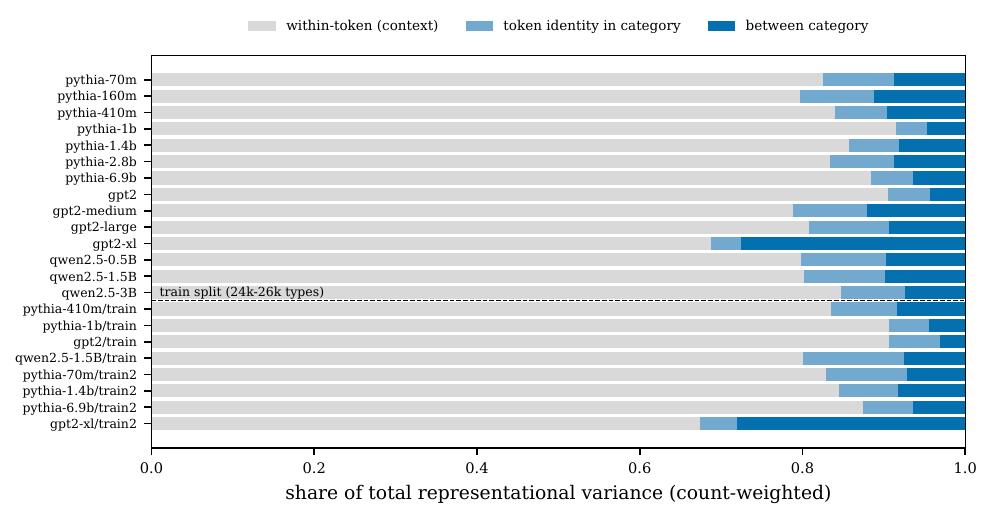}
\caption{The allocation of representational variance (count-weighted shares, $K{=}10$
illustrative partition) across 14 models at validation coverage and 8 train-split
extractions at 24k--26k types. Within-token context variability dominates everywhere;
macro-category structure is a thin slice; the picture is stable across families, a 100x
parameter range, and a 45x change in covered types.}
\label{fig:v2allocation}
\end{figure}

\paragraph{The allocation picture.}
Table~\ref{tab:v2allocation} and Figure~\ref{fig:v2allocation} show the result. Within-token context variability
dominates everywhere: 79--91\% of total variance on the validation subset outside
GPT-2 XL (0.69 there), and the same band at 24k--26k types. Token identity within
category carries 4--13\%, and macro
category structure 4.3--12\% (28\% for GPT-2 XL, whose centroid degeneracy we flag
throughout). Three observations. First, the allocation is nearly invariant to a 45x
increase in the number of covered types, so it is not an artifact of the frequent-token
subset. Second, there is no consistent scale trend within any family: the
between-category share moves with $\log$ parameters at Spearman $\rho{=}-0.46$
($p{=}0.29$) in Pythia, $+0.80$ in GPT-2, and $-0.50$ in Qwen2.5, and the standard
$\alpha{=}0$ CDNV is likewise non-monotone (Pythia $\rho{=}+0.64$). Across a 100x
parameter range the allocation stays in a narrow band: whatever sets it, it is not
capacity. Third, the same protocol applied over training checkpoints
(Section~\ref{sec:v2dynamics}) shows the allocation is far from static in time, which is
precisely why the cross-sectional stability at the end of training calls for an
explanation. The theory of Section~\ref{sec:v2theory} supplies a candidate: the
allocation tracks the information structure of the prediction task, whose decomposition
across levels is a property of the corpus, not of the model.

\paragraph{Which controls survive.}
Two standard positive controls behave: untrained checkpoints show an order of magnitude
less category structure than trained ones, and exogenous part-of-speech categories
reproduce the qualitative allocation without any clustering step. Two familiar
magnitude claims do not survive. The k-means-versus-shuffled-partition separations that
we and others have quoted (hundreds-fold under dimensional metrics) are reproduced at
matched to within a factor of 1.1--1.75 by covariance-matched Gaussian point clouds
with no linguistic structure whatsoever; that null is strict by construction (it
inherits the variance directions the categories themselves generate), so the excess is
a lower bound on learned structure, but it remains a factor, not orders of magnitude.
And the exogenous part-of-speech scaling control that our earlier version leaned on is
itself an artifact of the $\alpha{=}2$ metric: under $\alpha{=}0$ the POS-grouped trend
is flat to reversed (Pythia $\rho{=}+0.46$, Qwen2.5 $\rho{=}+1.0$).
Appendix~\ref{app:v2pitfalls} gives the full accounting. The constructive consequence:
claims in this paper rest on dimensionless shares, explicit nulls, and cross-model sign
consistency, not on ratio magnitudes.

\section{Theory: what geometry next-token prediction allows}
\label{sec:v2theory}

Three results follow from one decomposition: mean-cosine ETF evidence is an accounting
identity; token-level weight decay turns next-token prediction into an imbalanced
category problem whose norms are ordered by type count; and context-dependent
within-category choice forces a floor on feature dispersion. We take them in that
order.

\paragraph{Setup.}
A vocabulary $V$ with $|V|=C$ is partitioned into categories $\{S_k\}_{k=1}^K$ with type
counts $|S_k|$ and occurrence masses $N_k$ ($N=\sum_k N_k$, priors $\pi_k = N_k/N$). The
model assigns logits $z_c = \mathbf{w}_c^\top \mathbf{h}$ with head rows
$\mathbf{w}_c \in \mathbb{R}^d$ and context features $\mathbf{h}_i \in \mathbb{R}^d$ for
contexts $i=1,\dots,N$ with next token $y_i$. The regularized objective in the
unconstrained features model (UFM) is
\[
L \;=\; -\frac{1}{N}\sum_{i=1}^N \log \frac{e^{\mathbf{w}_{y_i}^\top \mathbf{h}_i}}
{\sum_{c} e^{\mathbf{w}_{c}^\top \mathbf{h}_i}}
\;+\; \frac{\lambda}{2}\Bigl(\sum_c \|\mathbf{w}_c\|^2 + \sum_i \|\mathbf{h}_i\|^2\Bigr).
\]
Write $\bar{\mathbf{w}}_k = \frac{1}{|S_k|}\sum_{c \in S_k} \mathbf{w}_c$ and
$\boldsymbol{\delta}_c = \mathbf{w}_c - \bar{\mathbf{w}}_{k(c)}$, so
$\sum_{c \in S_k} \boldsymbol{\delta}_c = 0$ for every $k$.

\subsection{Mean pairwise cosine is fixed by centering, not by learning}

\begin{lemma}[Centering identity]\label{lem:centering}
\textup{[PROVED]}
Let $\mathbf{v}_1,\dots,\mathbf{v}_K \in \mathbb{R}^d$, let
$\tilde{\mathbf{v}}_k = \mathbf{v}_k - \frac{1}{K}\sum_j \mathbf{v}_j$. Then
\[
\frac{1}{K(K-1)}\sum_{j \neq k} \langle \tilde{\mathbf{v}}_j, \tilde{\mathbf{v}}_k\rangle
\;=\; -\,\frac{1}{K-1}\cdot \frac{1}{K}\sum_k \|\tilde{\mathbf{v}}_k\|^2 .
\]
In particular, if all centered norms are equal, the mean pairwise cosine equals
$-1/(K-1)$ exactly, for any configuration whatsoever.
\end{lemma}

\begin{proof}
$\|\sum_k \tilde{\mathbf{v}}_k\|^2 = 0$ expands to
$\sum_k \|\tilde{\mathbf{v}}_k\|^2 + \sum_{j\neq k}\langle \tilde{\mathbf{v}}_j,
\tilde{\mathbf{v}}_k\rangle = 0$; divide by $K(K-1)$. With equal norms $r$, each inner
product term is $r^2 \cos\theta_{jk}$ and the display gives
$\overline{\cos\theta} = -1/(K-1)$.
\end{proof}

\begin{remark}\label{rem:centering}
Consequence for measurement: after centering, a mean pairwise cosine near $-1/(K-1)$
carries no information about learned structure beyond approximate equinormness; it is an
accounting identity. Equiangularity must instead be assessed through the dispersion of
pairwise cosines and the structure of their deviations (Section~\ref{sec:v2frame}). When
centering is done at a weighted grand mean and the partition is unbalanced, the identity
holds approximately, with error controlled by the imbalance; we quantify this in the
appendix. This retroactively voids the mean-cosine ``simplex ETF'' evidence in our own
earlier report and, we suspect, similar readings elsewhere.
\end{remark}

\subsection{Exact decomposition and the reduction to a size-weighted category problem}

\begin{proposition}[Exact split of loss and decay]\label{prop:split}
\textup{[PROVED]}
For any $\mathbf{W}, \mathbf{H}$ and any partition $\{S_k\}$:
\begin{enumerate}
\item[(i)] $\displaystyle\sum_c \|\mathbf{w}_c\|^2 \;=\; \sum_k |S_k|\,\|\bar{\mathbf{w}}_k\|^2
\;+\; \sum_c \|\boldsymbol{\delta}_c\|^2$.
\item[(ii)] The category logit obeys
$\log \sum_{c \in S_k} e^{\mathbf{w}_c^\top \mathbf{h}}
= \bar{\mathbf{w}}_k^\top \mathbf{h} + b_k(\mathbf{h})$ with
$b_k(\mathbf{h}) = \log \sum_{c \in S_k} e^{\boldsymbol{\delta}_c^\top \mathbf{h}}$.
\item[(iii)] The cross-entropy splits exactly as $L_{\mathrm{CE}} = L_{\mathrm{inter}} +
L_{\mathrm{intra}}$, where $L_{\mathrm{inter}}$ is the $K$-class cross-entropy on logits
$\bar{\mathbf{w}}_k^\top \mathbf{h} + b_k(\mathbf{h})$ and
\[
L_{\mathrm{intra}} = -\frac{1}{N}\sum_i \log
\frac{e^{\boldsymbol{\delta}_{y_i}^\top \mathbf{h}_i}}
{\sum_{c \in S_{k(y_i)}} e^{\boldsymbol{\delta}_c^\top \mathbf{h}_i}},
\]
which depends on the residual rows only.
\end{enumerate}
\end{proposition}

\begin{proof}
(i) is the orthogonal (bias-variance) decomposition using
$\sum_{c\in S_k}\boldsymbol{\delta}_c = 0$. (ii) factor $e^{\bar{\mathbf{w}}_k^\top
\mathbf{h}}$ out of the sum. (iii) write $p(y \mid \mathbf{h}) = p(S_{k(y)} \mid
\mathbf{h})\, p(y \mid S_{k(y)}, \mathbf{h})$ and note
$p(y \mid S_k, \mathbf{h}) = e^{\boldsymbol{\delta}_y^\top \mathbf{h}} / \sum_{c \in S_k}
e^{\boldsymbol{\delta}_c^\top \mathbf{h}}$ because the shared term
$\bar{\mathbf{w}}_k^\top \mathbf{h}$ cancels.
\end{proof}

Note what (i) already says: token-level weight decay does not act uniformly on category
structure. It penalizes the category mean row with an effective coefficient
$\lambda_k = \lambda\,|S_k|$, proportional to the category's TYPE COUNT, while occurrence
mass $N_k$ enters through the loss term. Categories are therefore doubly weighted, by
types in the regularizer and by mass in the risk, and the two weightings are empirically
distinguishable (Section~\ref{sec:v2frame}).

\begin{assumption}[Context-free within-category readout]\label{asm:cf}
$\boldsymbol{\delta}_c^\top \mathbf{h}_i = \beta_c$ for all training contexts $i$, i.e.,
the model's within-category conditionals $p(c \mid S_k, \mathbf{h})$ do not vary with
context.
\end{assumption}

Assumption~\ref{asm:cf} is an idealization, and Section~\ref{sec:v2law} shows it is
quantitatively false in trained models: within-category choice is context-dependent,
and that failure is exactly what Proposition~\ref{prop:floor} prices. The reduction
below is the between-category skeleton of the problem, not a description of real
conditionals; the floor is the correction term.

\begin{theorem}[Exact reduction, uniform within-category conditionals]\label{thm:reduction}
\textup{[PROVED; proof in Appendix~\ref{app:v2proofs}]}
Suppose Assumption~\ref{asm:cf} holds and the empirical within-category conditionals are
uniform, $\hat p(c \mid S_k) = 1/|S_k|$. Then the optimal residuals vanish
($\boldsymbol{\delta}_c = 0$), the offsets are exactly $b_k = \log|S_k|$,
$L_{\mathrm{intra}} = \sum_k \pi_k \log |S_k|$ is a constant, and minimizing $L$ over
$(\mathbf{W}, \mathbf{H})$ is \emph{exactly} the $K$-class UFM with class priors
$\pi_k$, per-class weight-decay coefficients $\lambda_k = \lambda\,|S_k|$ on the
category rows, and fixed per-class offsets $\log|S_k|$.
\end{theorem}

\begin{proposition}[General conditionals: the reduction up to an explicit coupling]
\label{prop:general}
\textup{[PROVED with the coupling exhibited; Appendix~\ref{app:v2proofs}]}
Without uniformity, realizing the loss-optimal constants $\{\beta_c\}$ costs residual
norm, and the minimal cost $C(\{\beta_c\}, \mathbf{H})$ depends on the span of the
features. The objective under Assumption~\ref{asm:cf} equals the reduced $K$-class
problem of Theorem~\ref{thm:reduction} (with offsets
$b_k = \log\sum_{c \in S_k} e^{\beta_c}$) plus $\frac{\lambda}{2}\,
C(\{\beta_c\}, \mathbf{H})$, and $C$ is the only channel through which residuals
interact with features. $C \ge 0$, and $C = 0$ exactly in the uniform case. Treating
$C$ as constant recovers the reduced problem as an approximation; we exhibit $C$ in
closed variational form in the appendix, estimate its magnitude on the trained models
at $10^{-6}$ to $10^{-3}$ of the residual-norm budget they actually deploy
(Appendix~\ref{app:reduction}), and leave sharp analytic control of its effect on the
optimum open.
\end{proposition}

\subsection{The geometry of the reduced problem}

\begin{proposition}[Exact $K{=}2$ solution: decay orders the norms, not mass]
\label{prop:k2}
\textup{[PROVED; Appendix~\ref{app:k2}]}
For $K = 2$, only the difference $\boldsymbol{\Delta} = \mathbf{u}_1 - \mathbf{u}_2$
enters the loss, and minimizing the weighted regularizer subject to fixed
$\boldsymbol{\Delta}$ gives $\mathbf{u}_1 = \frac{a_2}{a_1 + a_2}\boldsymbol{\Delta}$,
$\mathbf{u}_2 = -\frac{a_1}{a_1 + a_2}\boldsymbol{\Delta}$: classifier norms satisfy
$\|\mathbf{u}_k\| \propto 1/a_k$ exactly, for every $\lambda$, independent of priors
and offsets. Feature norms are ordered by the offsets (the larger-offset class takes
the smaller feature norm) at finite $\lambda$, an effect of order
$1/\log(1/\lambda)$ that vanishes in the strict SVM limit; priors set the overall
scale and touch the ratios only at second order. In our application ($a_k = \lambda
|S_k|$, $b_k$ increasing in $|S_k|$), classifier norms therefore decrease in category
TYPE COUNT at first order, feature norms follow through the offsets as a weaker
finite-$\lambda$ effect, and occurrence mass is secondary for both.
\end{proposition}

\begin{theorem}[General $K$: twisted SELI geometry]
\label{conj:seli}
\textup{[PROVED; Appendix~\ref{app:twisted}]}
In the $\lambda \to 0$ limit, the global minimizers of the reduced problem lie in the
twisted simplex-encoded-labels-interpolation (SELI) family: the joint Gram is the
two-block SVD factorization of \citet{thrampoulidis2022imbalance} built on the
weighted logit matrix $Z_w = A^{1/2}\,\bar Z^{*}\,D_n^{1/2}$, with the classifier
block conjugated by $A^{-1/2}$ (so $U^\top U = A^{-1/2}(V^\top V)A^{-1/2}$), and the
optimal logit matrix is the oblique projection
\[
\bar Z^{*} \;=\; I_K \;-\; \mathbf{1}_K\,\mathbf{a}^\top / \mathrm{tr}(A),
\]
which reduces to SELI's centering matrix $I - \mathbf{1}\mathbf{1}^\top/K$ at equal
decay. The variational characterization is proved in full: membership, primal
feasibility, the dual-structure lemma, the zero-duality-gap identity, and the final
sign condition on the dual multipliers (the twisted analogue of SELI's own sign
lemma), which we establish in closed form via an orthogonal-projection factorization
of the optimal weighted logit. The passage from finite-$\lambda$ minimizers to this
margin problem follows SELI's regularization-path argument under the
reparameterization, with the offsets exiting the effective margins at rate
$1/\log(1/\lambda)$, the rate proved exactly at $K{=}2$ and inherited rather than
re-derived at general $K$ (Appendix~\ref{app:twisted}).
\end{theorem}

Existing theorems do not cover this regime: SELI is proved for a single global
$\lambda$ (their Appendix C.3 permits only a global rescaling) and no offsets, and the
cross-entropy analysis of \citet{hong2024neural} likewise assumes homogeneous
regularization; reparameterizing $\mathbf{u}_k = a_k^{-1/2}\mathbf{v}_k$ shows the
$\lambda \to 0$ limit is precisely the weighted-norm margin problem that scope
excludes, and Theorem~\ref{conj:seli} closes it. Three
consequences deserve emphasis. First, the naive import of the vanilla-SELI ``majority
classes take larger classifier norms'' ordering to the reduced problem is WRONG: a
three-class counterexample with priors and decay weights in opposition shows the
classifier ordering follows $1/a_k$, not mass, and not the ratio $n_k / a_k$
(Appendix~\ref{app:k2}). An earlier draft of this work conjectured the mass ordering;
Proposition~\ref{prop:k2} corrects it. Second, the theorem makes the ordering claim
precise: $1/a_k$ enters as an exact conjugation factor on the classifier Gram, so
decay (type count) is the first-order orderer with a secondary mass modulation; the
strict law $\|\mathbf{u}_k\| \propto 1/a_k$ is exact only at $K{=}2$, and at
$K \ge 3$ the prediction is ordinal (limit-exact draws with decay and mass log-uniform
across two decades: median Spearman between $\|\mathbf{u}_k\|$ and $1/a_k$ of $+1.0$
at both $K{=}3$ and $K{=}4$, positive in over 99\% of draws; the shipped sweep script
reproduces these, and narrower decay ranges lower only the strict-ranking rates).
Third, this assigns the two empirical channels of
Section~\ref{sec:v2frame} to different mechanisms: the type-count channel is
first-order theory; the mass channel is not, and is instead the token-level frequency
structure that section controls for. Both assignments are confirmed empirically, on
feature centroids and directly on the models' own unembedding rows (24 of 24 cells;
Section~\ref{sec:v2frame}).

\subsection{An information floor on within-category dispersion}

If within-category choice requires context, Assumption~\ref{asm:cf} must fail, and the
failure is quantitative: realizing context-varying conditionals through a linear head
forces feature dispersion along the residual readout directions.

\begin{proposition}[Information floor, binary case]\label{prop:floor}
\textup{[PROVED]}
Fix a category $S_k = \{a, b\}$ and write $q(\mathbf{h}) = p(a \mid S_k, \mathbf{h}) =
\sigma\bigl(\mathbf{u}^\top \mathbf{h}\bigr)$ up to a constant, where $\mathbf{u} =
\boldsymbol{\delta}_a - \boldsymbol{\delta}_b$, $\|\mathbf{u}\| \le R$. Let contexts be
drawn from the conditional context distribution of the category, let $\bar q =
\mathbb{E}[q(\mathbf{h})]$, and let
$I_k = \mathbb{E}\,\mathrm{KL}\bigl(q(\mathbf{h}) \,\|\, \bar q\bigr)$ be the
model-realized conditional mutual information $I(c;\mathbf{h} \mid S_k)$. Then the
within-category feature dispersion $D_k = \mathbb{E}\|\mathbf{h} -
\mathbb{E}[\mathbf{h}]\|^2$ satisfies
\[
D_k \;\ge\; \frac{16\,\bar q(1-\bar q)}{R^2}\; I_k .
\]
\end{proposition}

\begin{proof}
Three inequalities chain. (1) The binary KL is bounded by chi-square:
$\mathrm{KL}(q\|\bar q) \le (q-\bar q)^2 / (\bar q(1-\bar q))$, hence
$\mathbb{E}(q-\bar q)^2 \ge \bar q(1-\bar q) I_k$ and, since $\bar q = \mathbb{E} q$,
$\mathrm{Var}(q) \ge \bar q(1-\bar q) I_k$. (2) $\sigma$ is $\tfrac14$-Lipschitz, so for
independent copies $\mathbf{h}, \mathbf{h}'$, $(q - q')^2 \le \tfrac{1}{16}
(\mathbf{u}^\top(\mathbf{h}-\mathbf{h}'))^2$; taking expectations,
$\mathrm{Var}(q) \le \tfrac{1}{16}\mathrm{Var}(\mathbf{u}^\top \mathbf{h})$. (3)
$\mathrm{Var}(\mathbf{u}^\top \mathbf{h}) = \mathbf{u}^\top \Sigma\, \mathbf{u} \le
\|\mathbf{u}\|^2 \operatorname{tr}(\Sigma) \le R^2 D_k$ with $\Sigma$ the context
covariance. Combining the three steps,
$R^2 D_k \ge \mathrm{Var}(\mathbf{u}^\top \mathbf{h}) \ge 16\,\mathrm{Var}(q) \ge
16\,\bar q(1-\bar q)\, I_k$.
\end{proof}

\begin{remark}[What the floor floors, and what it does not]\label{rem:floor}
Step (3) passes through the trace by Cauchy--Schwarz and is deliberately loose, so the
proved inequality lower-bounds the \emph{total} within-category dispersion
$D_k = \mathrm{tr}(\Sigma)$. It thus establishes that context-dependent within-category
choice requires within-category dispersion, but it does not by itself say which
component of that dispersion carries the requirement. That component question is
settled empirically, not here. Section~\ref{sec:v2law} also tests the inequality
directly on the readout-margin variance it controls, without the trace step: it holds
in all 286 within-category pairs measured and that variance tracks $I_k$. The total is a poor place to look: it also stores
context for the rest of the predictive distribution and can be large for reasons
unrelated to $I_k$, and indeed it does not track $I_k$ at all
(Section~\ref{sec:v2law}, negative control). The component that does track $I_k$ is the
token-identity dispersion, the spread of token class means within a category. We read
this as the required dispersion surfacing in the identity channel: the readout
direction $\mathbf{u} = \boldsymbol{\delta}_a - \boldsymbol{\delta}_b$ is a difference
of classifier rows, and classifier rows co-move with token feature means in these
models (Section~\ref{sec:v2frame}), so dispersion forced along $\mathbf{u}$ appears
between token means. This localization is consistent with the floor but is not a
consequence of it; a component-level bound is the natural next theorem. The general-$K$
statement, with constants degrading in $K$ through the softmax Lipschitz constant and a
pairwise chi-square bound, is stated in the appendix as a conjecture with the same
proof architecture.
\end{remark}

\begin{remark}[Relation to the softmax bottleneck]
\citet{yang2018breaking} bound the GLOBAL rank of realizable log-probability matrices by
the head dimension. Proposition~\ref{prop:floor} is a local, per-category refinement in
an information metric: it does not limit what the model can express globally, but prices
context-dependent within-category choice in units of feature dispersion. It thereby
explains why within-class variability in language models cannot fully collapse
\citep{wu2024linguistic}: the residual variability is not unfinished collapse but the
storage cost of conditional information.
\end{remark}

\section{The information law: identity dispersion tracks conditional information}
\label{sec:v2law}

Proposition~\ref{prop:floor} prices context-dependent within-category choice in units of
feature dispersion along readout directions, and Remark~\ref{rem:floor} identifies the
component of dispersion the price should load on: the spread of token class means within
a category (identity dispersion), not total within-category variance, because a
context-free preference over category members can be realized by head rows alone. This
section tests that prediction per category.

\paragraph{Estimating per-category conditional information.}
For each model we run one inference pass over the extraction corpus and accumulate, for
every category $S_k$ of every partition, the model's within-category conditional
$q(c \mid S_k, \mathbf{h}) \propto \exp(\mathbf{w}_c^\top \mathbf{h})$, $c \in S_k$,
restricted to covered types. The model-realized conditional information is
\[
\hat I_k \;=\; H\bigl(\bar q_k\bigr) \;-\; \mathbb{E}_{\mathbf{h}}\,
H\bigl(q(\cdot \mid S_k, \mathbf{h})\bigr),
\qquad \bar q_k = \mathbb{E}_{\mathbf{h}}\, q(\cdot \mid S_k, \mathbf{h}),
\]
which is nonnegative by concavity and equals $I(c; \mathbf{h} \mid S_k)$ under the
model's conditionals with the empirical context distribution. We compute a realized
estimator (contexts whose true next token falls in $S_k$) and a posterior-weighted
estimator (all contexts, weighted by $p(S_k \mid \mathbf{h})$); they agree closely, and
a variant using corpus counts for the marginal term guards against
self-reference. Identity dispersion $B_k$ (the spread of token means \emph{between}
category members), within-token variance $W_k$, and their total $V_k = B_k + W_k$ come
from the class statistics of Section~\ref{sec:v2measurement} (count-weighted). Models: GPT-2, Pythia-410M, Pythia-1.4B, Qwen2.5-1.5B. Partitions:
$k$-means $K \in \{10, 20, 50\}$ and part-of-speech; categories with fewer than 200
realized positions are excluded.

\begin{figure}[t]
\centering
\includegraphics[width=\textwidth]{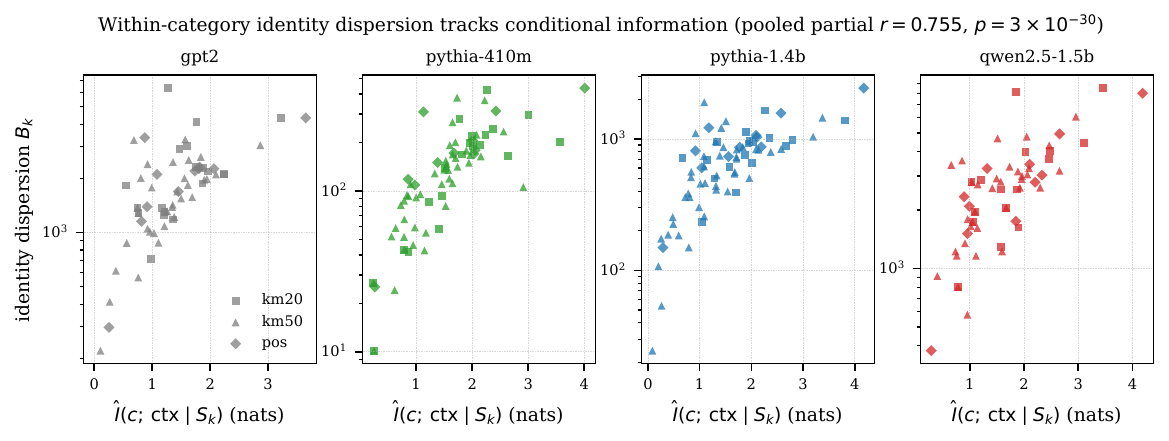}
\caption{Identity dispersion $B_k$ (log scale) against model-realized conditional
information $\hat I_k$, per category, for four models and three partitions. The
relationship is positive in every model-partition combination tested (16/16 including
$K{=}10$, not shown for clarity); pooled within-model-rank partial correlation
$r = 0.755$ ($p = 3\times 10^{-30}$, $n = 157$ categories at $K{=}50$) controlling for
marginal entropy, category mass, and type count.}
\label{fig:v2gatea}
\end{figure}

\begin{table}[t]
\centering
\small
\caption{Spearman correlation of identity dispersion $B_k$ with $\hat I_k$ across
categories, per model and partition ($n$ = number of categories passing the position
filter). Every cell is positive; cells with more categories are systematically stronger.}
\label{tab:v2law}
\begin{tabular}{lrrrr}
\toprule
\textbf{Model} & $K{=}10$ & $K{=}20$ & $K{=}50$ & \textbf{POS} \\
\midrule
GPT-2         & $+0.75$ & $+0.58$ ($n{=}19$) & $+0.67$ ($n{=}35$) & $+0.71$ ($n{=}11$) \\
Pythia-410M   & $+0.49$ & $+0.82$ ($n{=}20$) & $+0.85$ ($n{=}44$) & $+0.81$ ($n{=}11$) \\
Pythia-1.4B   & $+0.52$ & $+0.70$ ($n{=}18$) & $+0.75$ ($n{=}44$) & $+0.79$ ($n{=}11$) \\
Qwen2.5-1.5B  & $+0.39$ & $+0.68$ ($n{=}16$) & $+0.62$ ($n{=}34$) & $+0.82$ ($n{=}11$) \\
\bottomrule
\end{tabular}
\end{table}

\paragraph{The law, and the placebos.}
Figure~\ref{fig:v2gatea} and Table~\ref{tab:v2law} give the central result. Identity
dispersion tracks conditional information in 16 of 16 model-partition combinations, at
Spearman $\rho = 0.58$--$0.85$ wherever the partition supplies more than ten categories.
The relationship survives rank-based partial correlation controlling simultaneously for
marginal within-category entropy, log category mass, and type count (pooled within-model
ranks: $r = 0.745$, $p = 4\times10^{-14}$, $n = 73$ at $K{=}20$; $r = 0.755$,
$p = 3\times10^{-30}$, $n = 157$ at $K{=}50$; posterior-weighted estimator with
least-squares partialing, released with the code; a precision-matrix recomputation
gives slightly larger values, so these figures are conservative). Three negative
controls (placebos) discriminate the
information account from blander alternatives. Total within-category variance $V_k$,
dominated by within-token context storage, shows no relationship with $\hat I_k$ (mixed
signs across models and partitions; at $K{=}10$, $-0.19$ to $+0.52$, and in the single
cell where it is nominally large, Qwen2.5 at $K{=}10$, it is inverted at every richer
partition): the floor loads on the identity component, as it should.
Within-token variance $W_k$ alone likewise shows nothing. And $B_k$ tracks conditional
information roughly twice as strongly as it tracks the marginal entropy
$H(c \mid S_k)$ (at $K{=}50$: $0.62$--$0.85$ versus $0.27$--$0.36$): what predicts
dispersion is not how many alternatives a category has but how much context is needed to
choose among them, which is exactly the distinction between head-expressible static
preferences and dispersion-requiring conditional ones.

\paragraph{The floored quantity, tested directly.}
Identity dispersion is a proxy for what Proposition~\ref{prop:floor} actually bounds,
which is the variance of the readout margin along $\mathbf{u} = \boldsymbol{\delta}_a -
\boldsymbol{\delta}_b$. That variance is $\mathrm{Var}(\text{logit}_a - \text{logit}_b)$
and is measurable directly from logits, with no hidden states and no layer choice, so we
test the theorem on the quantity it controls rather than on the proxy. For the dominant
within-category token pair of each category (286 pairs across the four models and the
$K{=}20$, $K{=}50$, and POS partitions), the proved inequality
$\mathrm{Var}(\mathbf{u}^\top\mathbf{h}) \ge 16\,\bar q(1-\bar q)\, I$ holds in every one
of the 286 pairs, and the margin variance tracks the pair's conditional information at
Spearman $0.94$ (per model $0.89$--$0.95$; permutation $p < 10^{-4}$). The bound is
genuine but loose: the median margin variance is $21\times$ its floor, so the
information requirement accounts for the ordering of the margin variances without
pinning their level. This links conditional information to the quantity the theorem
bounds; the further link to identity dispersion is the head-row and feature-mean
co-movement of Section~\ref{sec:v2frame}.

\paragraph{The loss-side anchor.}
The same inference pass yields the exact decomposition of cross-entropy into
between-category and within-category terms (Proposition~\ref{prop:split}(iii)). The
within-category share is 0.59--0.64 at $K{=}10$ across the four models, 0.43--0.53 at
$K{=}50$, and 0.46--0.48 under part-of-speech categories: the majority of next-token
prediction, at coarse granularity, is within-category choice. The variance-dominant
component of the geometry (Section~\ref{sec:v2measurement}) is thus also the
loss-dominant component of the task, which is what an allocation-follows-information
account requires.

\paragraph{Extremes are informative.}
Categories where context nearly determines the token (conditional entropy
$\mathbb{E}\,H(q \mid \mathbf{h})$ of 0.2--0.6 nats against marginal entropies of 2--4)
carry the largest identity dispersion and the largest identity share of their category
variance (up to 0.29): determiner-like and sentence-initial categories are the clearest
cases. High-frequency function-word categories combine large $\hat I_k$ with moderate
dispersion. The pattern is monotone through the middle of the range; per-category tables
for all models appear in the released artifacts.

\paragraph{Robustness and scope.}
The weakest cells are $K{=}10$ (ten points per model); inference there rests on
cross-partition and cross-model consistency rather than any single correlation. The
pooled $p$-values above are not exact tests, and we do not read them as such: models
within a family share tokenizers, and categories overlap in membership across models,
so the pooled categories are not independent draws. Three checks address the two
concerns this raises, non-independence and self-reference, without relying on the
pooled figure. First, a permutation test that shuffles conditional information within
each model, which preserves every marginal and the within-model dependence structure,
still rejects: the observed pooled partial correlation ($r = 0.70$) is not reached in
any of 20{,}000 permutations ($p < 5\times10^{-5}$). Second, the law does not depend on
reading information off the model's own head. Replacing the model posterior with a
model-free corpus-count estimate of conditional information leaves the partial
correlation essentially unchanged ($r = 0.68$ pooled; $0.65$--$0.71$ per model). Third,
and most directly against a within-model logit-variance tautology: on the exogenous
part-of-speech partition, whose categories are comparable across models, one model's
conditional information predicts a \emph{different} model's identity dispersion in all
twelve ordered model pairs (median Spearman $0.76$). A quantity that is merely a
restatement of one model's geometry cannot predict another model's geometry across a
change of architecture and tokenizer. The estimator remains model-based by design in
the main analysis (the floor concerns information the model actually realizes); these
checks show the correlation is not an artifact of that choice. Two further checks move
the categories and the corpus outside the model's own geometry and extraction
distribution. Categories from a semantic ontology: mapping covered tokens to WordNet
supersenses (25 categories, defined with no reference to the representations)
preserves the law in all four models (Spearman $0.64$--$0.83$; $0.49$--$0.64$ partial
given marginal entropy). And corpus transfer: re-running the identical pipeline on a
Pile sample with partitions held fixed leaves the allocation essentially unchanged
(context $83$--$87\%$, category $6$--$10\%$ at $K{=}10$, within a few points of the
same pipeline's WikiText values) and preserves the law in 12 of 12 model-partition
cells at $K \ge 20$ (median $\rho = 0.61$, against $0.73$ on WikiText; the
ten-category cells are noise-limited on both corpora). Nor is the law an artifact
of the layer: recomputing the variance components at the final layer (same partitions,
same information estimates, which are layer-free) preserves the law in three of the
four models ($\rho = 0.51$--$0.86$ across km20, km50, and POS); the exception is
GPT-2's final layer, whose geometry is dominated by a few extreme-magnitude
dimensions, a documented pathology of that model's last layer. Treating the model as
the unit of inference, a random-effects meta-analysis over per-model slopes (rank
scale, same three controls) gives all four slopes positive with pooled slope $1.09 \pm
0.10$ ($z = 11.0$; full vocabulary: $1.20 \pm 0.11$, $z = 10.6$) and small
between-model heterogeneity ($\tau^2 \le 0.04$); a mixed-effects fit with a random
information slope by model agrees. The four models span
three families, tied and untied heads, and a 12x parameter range.

\paragraph{The law at full vocabulary.}
Re-running the estimator on the train corpus (50{,}000 sequences, 25.6M positions) with
partitions over the full covered vocabulary (24k--26k types per model, the overcomplete
regime with $C/d \in [12, 34]$) reproduces the law at roughly 45x the type coverage
and 100x the sequences: identity
dispersion tracks conditional information in 12 of 12 further model-partition
combinations, and the pooled within-model-rank partial correlation rises to
$r = 0.773$ ($p = 6\times10^{-41}$, $n = 200$ categories at $K{=}50$; identity share:
$r = 0.750$). The placebos persist: total variance is sign-inconsistent across models
($-0.72$ to $+0.83$), and marginal entropy predicts $B_k$ far more weakly than
$\hat I_k$ does ($-0.22$ to $+0.54$). At this scale the dimensionless identity share
$B_k / V_k$ is the sturdier per-model statistic: at $K{=}10$ all four models exceed
$\rho = 0.68$, whereas raw $B_k$ has one weak cell (Pythia-1.4B, $+0.31$) driven by a
single very-high-variance function-word category. The loss-side anchor also
strengthens with coverage: at full vocabulary, 73--75\% of cross-entropy is resolved
within $K{=}10$ categories (58--62\% at $K{=}50$).

\section{The centroid frame: type-count ordering and its attribution}
\label{sec:v2frame}

Theorem~\ref{thm:reduction} and Proposition~\ref{prop:general} say the between-category
problem is an imbalanced $K$-class problem doubly weighted by mass (in the risk) and
type count (in the decay and offsets). What geometry that produces is not a matter of
importing vanilla imbalanced-collapse results
\citep{thrampoulidis2022imbalance, fang2021exploring, hong2024neural}: our exact
$K{=}2$ solution (Proposition~\ref{prop:k2}) shows the decay and offset channels
dominate, so at first order both classifier and feature norms decrease in category
TYPE COUNT, with occurrence mass entering only at second order. Since type count and
mass are correlated across real categories, both orderings appear in raw data; the
work of this section is to separate them. Because Lemma~\ref{lem:centering} retires
mean-cosine evidence, norms and structured deviations are the right observables. We
test across all 14 models and then try to break the result with the hardest null we
could construct. The cleanest test is on the models' own unembedding rows, and it
passes in every cell (24 of 24; end of section); the centroid-frame evidence below is
weaker and needs a frequency control, which is the work of this section.

\paragraph{The ordering holds with one sign in 14 of 14 models.}
For each model, we correlate centered mass-weighted centroid norms with log category
mass across the $K{=}10$ categories (Figure~\ref{fig:v2frame}). The correlation is negative in every model
(Spearman $-0.38$ to $-0.83$; sign test $p = 1.2\times10^{-4}$), and the pairwise-cosine
analogue (heavier pairs more antipodal) is likewise negative in every model. Against
size-matched shuffled partitions the learned coupling is stronger in direction in all
models, though only a minority of models individually clear the 10th percentile of the
shuffle null: as elsewhere in this paper, the evidence is cross-model sign consistency,
not per-model significance.

\begin{figure}[t]
\centering
\includegraphics[width=\textwidth]{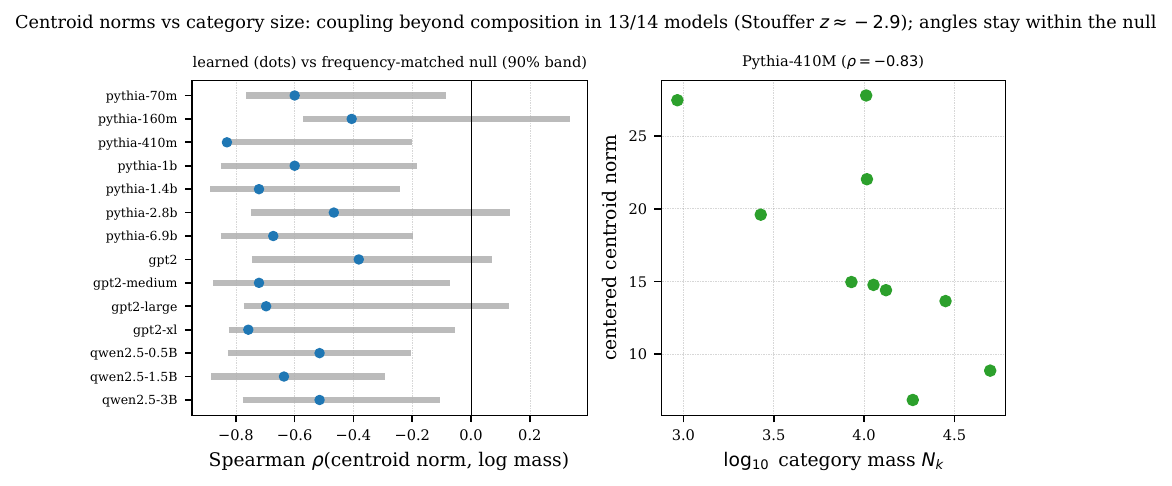}
\caption{Left: per-model Spearman correlation of centered centroid norm with log
category mass (dots) against the frequency-matched shuffle null (gray: mean and 90\%
band). The coupling is negative in 14/14 models and below its null mean in 13/14; the
evidence is the systematic shift (Stouffer $z \approx -2.9$). Right: the strongest
single model.}
\label{fig:v2frame}
\end{figure}

\paragraph{The hard null: frequency-matched shuffles.}
A mundane mechanism could produce the same sign: if frequent \emph{tokens} individually
sit closer to the grand mean, any grouping inherits a negative mass-norm relation from
composition alone. The token-level relation is real in all 14 models (Spearman of
centered class-mean norm against log frequency: $-0.25$ to $-0.38$; compare the
frequency-norm distortions long documented for embedding spaces,
\citealp{gao2019representation}). We therefore re-test against frequency-matched
shuffles, which permute category labels within log-frequency quartiles, preserving each
category's size and frequency profile and hence its mass.

This null reproduces roughly
70\% of the coupling (mean learned $\rho = -0.61$ against a null mean of $-0.44$). The
residual is individually inconclusive in any one model but systematic across them: the
learned coupling is more negative than its null mean in 13 of 14 models, and combining
per-model percentiles by Stouffer's method gives $z \approx -2.9$
($p \approx 1.9\times10^{-3}$, one-sided); as with the law section's pooled figures,
family-shared tokenizers make this a consistency statistic rather than an exact test.
The pairwise-cosine coupling, by contrast, sits squarely inside the frequency-matched
null (mean percentile 0.51) and we do not claim it. The frame claim of this paper is
therefore deliberately scoped: type-count-ordered norms as the confirmed first-order
prediction (strongest on the head rows below); a residual mass coupling beyond
composition, at modest strength; and, for angular structure, nothing beyond the null.

\paragraph{A conjectural second reading, subordinate to the control.}
We control for the token-level frequency-norm relation as a confound, and the
norm-ordering claim survives; nothing here rests on more. We note only, without
evidentiary weight, that this confound runs in the same direction imbalanced-collapse
theory predicts when tokens are treated as the finest categories, so one size-weighted
mechanism at every partition level could in principle produce both the confound and
the signal; separating coincidence from nesting is now a concrete test, since
Theorem~\ref{conj:seli} supplies the functional form, and fitting it to the measured
frames is scheduled for the conference version.

\paragraph{Two channels, two mechanisms.}
Pooling within-model ranks across the 140 categories, centroid norm relates to log
type count controlling for log mass at $r = -0.41$ ($p = 5\times10^{-7}$), and to log
mass controlling for log type count at $r = -0.51$ ($p = 2\times10^{-10}$). The
corrected theory assigns these to different mechanisms: the type-count partial is the
channel Proposition~\ref{prop:k2} predicts at first order (decay and offsets), while
the mass partial is NOT first-order theory and is precisely the channel the
frequency-matched null above addresses (token-level frequency structure aggregating to
categories). An earlier draft read both channels as one prediction under an imported
mass ordering; the $K{=}2$ solution shows that import was wrong, and the present
assignment is the operative one. These remain observational partials on ten categories
per model, not an identification.

The head-row side of Proposition~\ref{prop:k2},
however, is directly testable from released model weights, and it holds: across the
four models, both coverage regimes, and $K \in \{10, 20, 50\}$, the correlation
$\rho(\|\bar{\mathbf{w}}_k\|, \log|S_k|)$ is negative in 24 of 24 cells (median
$-0.54$), the type-count partial given mass is negative in 24 of 24 (median $-0.50$),
and the mass partial given type count is sign-unstable (positive in a third of cells):
the head rows reproduce the same type-count-first-order, mass-second-order asymmetry
as the feature centroids. The log-log slopes ($-0.03$ to $-0.36$) are far shallower
than the $K{=}2$-exact $-1$, as expected for networks that are not at the UFM optimum;
the tested predictions are the ordering and the channel asymmetry, not the exponent.

\section{Dynamics: crystallize, overshoot, reallocate, partially recover}
\label{sec:v2dynamics}

The allocation of Section~\ref{sec:v2measurement} is stable across final checkpoints but
is anything but static in training time. We track the three variance shares over public
pretraining checkpoints under a fixed partition, so that trajectory changes reflect
representational change rather than re-clustering: categories are defined once by
$k$-means on the final checkpoint's class means (with an exogenous part-of-speech
partition as a clustering-free control), and shares are recomputed at every checkpoint
on the common covered-token set.

\paragraph{Checkpoint hygiene.}
Loading intermediate checkpoints through revision tags silently returned identical
weights for one model in our earlier pipeline (all eight Pythia-2.8B ``checkpoints''
were the final model; the flat trajectory was the tell). The OLMo-2 checkpoints are
therefore downloaded per revision into separate directories
with per-checkpoint weight fingerprints (all 14 distinct); the earlier Pythia dumps
predate this guard and are screened by trajectory non-constancy, the test that exposed
2.8B. Pythia-2.8B is excluded pending re-extraction; Pythia-6.9B
was extracted under the smaller covered-token protocol and is shown for shape only.

\begin{figure}[t]
\centering
\includegraphics[width=\textwidth]{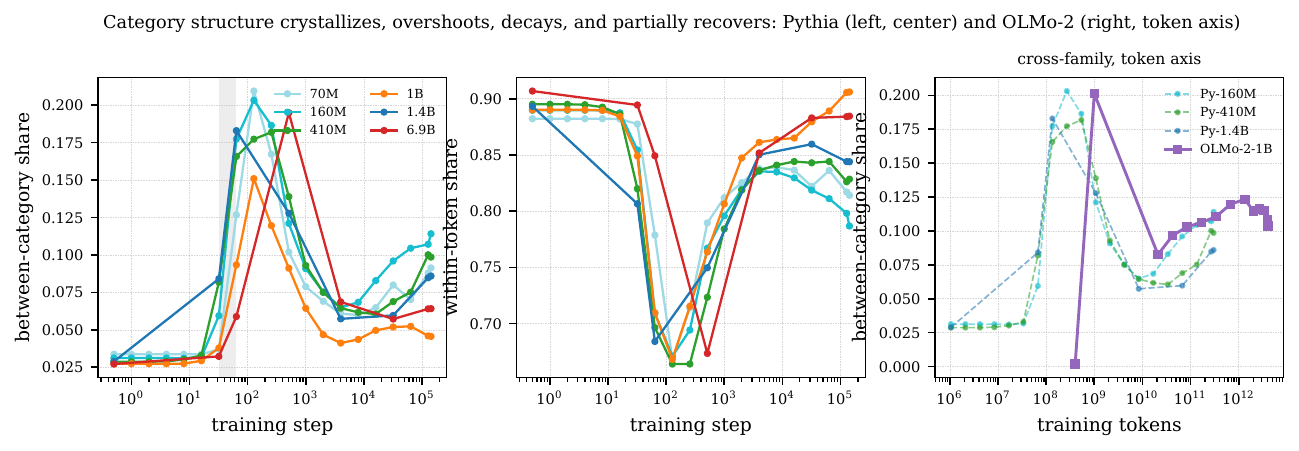}
\caption{Variance-share trajectories under fixed final-checkpoint partitions
(count-weighted, common token sets). Left: Pythia between-category share, onset window
(steps 32--64) shaded. Center: Pythia within-token share. Between-category structure
overshoots within the first thousand steps, decays, and partially recovers in four of
six sizes (Pythia-1B, orange, excepted). Right: OLMo-2-1B reproduces the overshoot, a
minimum near 21B tokens, and late recovery.}
\label{fig:v2dynamics}
\end{figure}

\paragraph{The Pythia trajectory.}
Figure~\ref{fig:v2dynamics} shows five phases, consistent across the sizes whose
checkpoint grids resolve them.
(1) \emph{Dormancy.} Through step 16 the between-category share sits at its
initialization value (about 0.03).
(2) \emph{Onset.} Between steps 32 and 64 it multiplies severalfold, at every size whose
grid covers this window; the onset step is scale-invariant to the resolution of the
grid.
(3) \emph{Overshoot.} It peaks at 0.15--0.21 within the first 64--512 steps, about two
to threefold its final value (1.8--3.3x across sizes); the within-token share
correspondingly dips to 0.66--0.68, its minimum over all of training.

(4) \emph{Reallocation.} The between-category share then decays over two orders of
magnitude in steps, reaching minima of 0.057--0.065 at steps 4k--32k, while the
within-token share re-expands past 0.83: the model, having crystallized coarse
structure, progressively reallocates variance to context.
(5) \emph{Partial recovery.} In four of six sizes the between-category share then rises
again by a factor of 1.5--1.8 (for example 160M: 0.065 at step 4k to 0.114 at step
143k), with a mild 1.12x at 6.9B on its coarse grid; Pythia-1B, an outlier on several
static measures as well, instead declines overall after its peak to 0.046 (with a
small non-monotonicity near step 64k). The
standard cluster-level CDNV ($\alpha{=}0$) tells the same story in its own units: best
collapse early (minima 3.1--4.9 at steps 256--1000), worse at the end (4.8--9.0). In
classification, the terminal phase collapses representations; here the terminal phase
partially un-collapses them, while validation loss improves throughout.

\paragraph{Why recovery is the expected ending.}
Under the floor of Proposition~\ref{prop:floor}, late-training gains that come from
context-dependent within-category choice must be paid for in dispersion along readout
directions; under the reduction, category structure itself remains load-bearing for the
between-category part of the loss. A monotone dissolution of category structure into
symmetric fine-grained geometry would require the conditional information carried at
the category level to vanish, and it does not. We frame this as consistency rather than
causal demonstration: no intervention is run, and the account earns its keep by
predicting where the non-monotonicity lives (the between-category share and the
identity component, not total variance).

\paragraph{Cross-family replication on OLMo-2 (pre-registered).}
We repeat the protocol on OLMo-2-0425-1B \citep{olmo2025two}, whose public grid spans
initialization, one checkpoint at roughly 1B tokens (inside the Pythia overshoot window
on the token axis; both families train at roughly 2.1M tokens per step), then a
geometrically spaced grid (roughly doubling from 21B to 1.3T tokens, then roughly
670B apart to 4T), plus a post-anneal release. Four criteria were fixed
before the run (thresholds hardcoded in the released analysis script, committed before
any checkpoint beyond initialization had been extracted; a dated criteria record ships
with the results): R1, elevation of the between-category share at the 1B-token checkpoint
over both initialization and the 21--42B range; R2, an interior minimum followed by
partial late recovery; R3, early dip and late re-expansion of the within-token share;
R4, an early minimum of cluster-level CDNV. The OLMo-2 grid cannot resolve the onset
(no public checkpoints between step 0 and step 300), so the onset claim remains
Pythia-only by design.

\paragraph{OLMo-2 verdict: three of four criteria pass.}
All 14 checkpoints pass the distinct-weights fingerprint guard. R1 (overshoot):
\textbf{pass}; the between-category share at the 1B-token checkpoint is 0.201, against
0.002 at initialization and 0.090 at 21--42B tokens. R2 (interior minimum with late
recovery): \textbf{pass}; the post-overshoot minimum is 0.083 at 21B tokens, the share
recovers to 0.124 by 1.3T tokens and stands at 0.108 at the 4T stage1 final (1.31x the
minimum). R3 (context re-expansion): \textbf{pass}; the within-token share dips from
0.988 at initialization to 0.725 at the overshoot and re-expands to 0.803. R4
(cluster-level CDNV U-shape): \textbf{fail against its pre-set bar}; the shape is
present (best 3.33 early, 4.05 at the end) but the 1.21 ratio is below the 1.3
threshold we fixed before the run, and we report it as such.

Three further
observations. First, the minimum's location matches Pythia on the token axis (21B
against Pythia's 8--67B; both families train at roughly 2M tokens per step), an
alignment that step-indexed comparisons would obscure
(Figure~\ref{fig:v2dynamics}, right). Second, beyond 1.3T tokens, a regime Pythia
never reaches, the recovery eases back by about 12\%, so the late reallocation is not
monotone at very long horizons. Third, the post-anneal release sits slightly below the
stage1 final (0.104 against 0.108), a first indication that annealing trims category
structure. The exogenous part-of-speech partition reproduces the same trajectory shape
throughout (peak 0.128, minimum 0.050, recovery to 0.084), so none of this depends on
the clustering step.

\paragraph{Relation to concurrent accounts.}
Two recent papers frame parts of this trajectory through different instruments.
\citet{zhao2026structure} document, in synthetic category languages, semantic structure
that emerges early and then dissolves monotonically into the symmetric geometry
predicted by classical collapse, and propose a normalization-based modification of the
unconstrained features model that enforces that ending. \citet{li2025tracing} trace
global spectral quantities (effective rank, eigenspectrum decay) over real Pythia and
OLMo checkpoints and report a non-monotonic three-phase trajectory, without category
structure or a collapse lens. The decomposition used here shows where the
non-monotonicity lives (the category level), and the real-checkpoint ending disagrees
with the synthetic extrapolation: category structure does not dissolve; it partially
returns. Aligning our share phases with their spectral phases on the same checkpoints
is a natural join and is left to the camera-ready timeline. The onset window itself
sits in the regime where models acquire low-order distributional statistics
\citep{belrose2024statistics} and precedes the induction-head window
\citep{olsson2022induction}; making those alignments quantitative is likewise deferred.

\section{Related work}
\label{sec:v2related}

\paragraph{Neural collapse and its imbalanced geometry.}
Neural collapse was documented by \citet{papyan2020prevalence} and explained in the
unconstrained features model \citep{zhu2021geometric, mixon2022neural}. Under class
imbalance the optimum is not a regular simplex: minority collapse
\citep{fang2021exploring}, the SELI characterization \citep{thrampoulidis2022imbalance},
and the cross-entropy analysis of \citet{hong2024neural} establish frequency-dependent
geometry, and \citet{jiang2024generalized} treat the many-class regime relevant to
language. All of these assume a single global regularization coefficient and no offsets.
Our reduction (Theorem~\ref{thm:reduction}) derives the coarse category problem from the
fine token loss and shows the induced regularization is per-category and
type-count-weighted, with mass entering separately through the risk; the geometry of
that doubly weighted problem is beyond current theorems; we solve it exactly at $K{=}2$
(Proposition~\ref{prop:k2}), where the answer diverges from a naive SELI import
(decay ordering dominates mass ordering), and in closed form for general $K$ with a
complete proof (Theorem~\ref{conj:seli}), with empirical support
on both feature centroids and head rows (Section~\ref{sec:v2frame}).

\paragraph{Neural collapse in language models.}
\citet{wu2024linguistic} measured NC-style quantities at the token level in causal
language models and found partial within-class collapse that never completes. Our floor
(Proposition~\ref{prop:floor}) offers an explanation with a changed reading: the
residual variability is not unfinished collapse but the storage cost of conditional
information, and its identity component tracks that information quantitatively
(Section~\ref{sec:v2law}). An earlier version of the present work claimed classical
collapse emerges at the category level; the instruments introduced here retire several
of those claims, and Appendix~\ref{app:v2pitfalls} audits them explicitly.

\paragraph{Transient structure and training-time geometry.}
\citet{zhao2026structure} show, in synthetic settings, semantic geometry that emerges
before collapsing monotonically into symmetry, with a normalization-modified
unconstrained features model predicting the symmetric ending; \citet{li2025tracing}
document non-monotonic spectral phases over real Pythia and OLMo pretraining without
category structure or a collapse framing. Section~\ref{sec:v2dynamics} measures the
category-resolved quantity on real checkpoints and finds an ending both accounts miss:
partial recovery. Related timing landmarks in the same checkpoint families include the
acquisition of low-order statistics \citep{belrose2024statistics} and the induction-head
phase change \citep{olsson2022induction}.

\paragraph{Concept geometry in large language models.}
\citet{park2024geometry} represent categorical concepts as simplices and hierarchies as
orthogonal direct sums in a causally whitened unembedding space, a static, observational
account with every category member weighted equally. The present account is
complementary and differs on each axis: it concerns hidden-state class means, is derived
from the loss and regularizer, is size- and mass-weighted (predicting distorted rather
than regular frames), and is dynamic.

\paragraph{Capacity limits of the softmax head.}
The softmax bottleneck \citep{yang2018breaking} bounds the global rank of realizable
log-probability matrices by the head dimension. Proposition~\ref{prop:floor} is a local
refinement in an information metric: a per-category dispersion floor in
$I(\mathrm{token}; \mathrm{context} \mid \mathrm{category})$, which explains a
qualitative fact the rank bound does not touch, namely which categories can and cannot
collapse.

\paragraph{Information-theoretic accounts of collapse.}
Recent work connects collapse to compression, arguing representations discard
task-irrelevant variability as collapse emerges \citep{grokking_ib_nc2025}. The floor
points the other way for the language setting: conditional information present in the
task places a lower bound that prevents full compression at the category level. The two
directions are compatible (compression of what is irrelevant, floors on what is not),
and the language-model allocation of Section~\ref{sec:v2measurement} quantifies where
each applies.

\section{Discussion and limitations}
\label{sec:v2discussion}

\paragraph{What the pieces add up to.}
Read together, the results replace a geometric narrative with an informational one.
The allocation of representational variance in language models is dominated by context
storage, stable across scale, and mirrored by the loss decomposition; the component of
within-category variance that theory says must exist (the identity dispersion needed to
realize context-dependent choice through a linear head) is the component that tracks
conditional information, category by category, in every model and partition tested; the
frame that category centroids form carries the type-count ordering the size-weighted
reduction predicts, confirmed directly on head rows in 24 of 24 cells, with the residual
mass coupling attributed to the controlled frequency confound; and over training,
category structure crystallizes,
overshoots, reallocates toward context, and partially returns. None of this requires
collapse to complete, and the floor explains why it cannot: the variance that refuses to
vanish is the variance doing the work.

\paragraph{Limitations, by pillar.}
\emph{Theory.} Theorem~\ref{conj:seli} is fully proved; the practical caveat is that
the finite-$\lambda$ approach to its limit geometry is slow (order
$1/\log(1/\lambda)$, quantified at $K{=}2$ and visible in the certificate's
finite-$\lambda$ deviations), so quantitative frame-fitting on trained models must
account for it.
Proposition~\ref{prop:floor} is proved for binary categories; the general-$K$ statement
is a conjecture with the same proof architecture and degrading constants. The reduction
requires the context-free readout assumption exactly where it is least true, and the
floor is precisely a statement about the cost of its failure; a unified treatment
(reduction with quantified leakage) is the natural next theorem.

\emph{Estimation.} The conditional-information estimator uses the model's own
conditionals, by design (the floor concerns realized information); the corpus-count
variant reproduces the law ($r = 0.68$) and the cross-model test breaks the remaining
self-reference, but all estimates still share one corpus. At $K{=}10$ each model
contributes only ten categories; the inference rests on cross-partition, cross-model
consistency.

\emph{Attribution.} The norm orderings are read off AdamW-trained weights. The decay
penalty is explicit in the objective and optimizer-independent, but adaptive
per-parameter scaling could in principle contribute to the same orderings along the
trajectory; separating the two requires decay-ablated training runs, which we have not
done, so the attribution of the ordering to the decay channel rests on the sign
structure it predicts (type count over mass), not on an optimizer intervention.

\emph{Scope.} The main analyses use English WikiText-103 at layer $L{-}1$ of
autoregressive models up to 6.9B parameters; a fixed-partition transfer check on a
Pile sample reproduces the allocation and the law, and a WordNet-supersense partition
reproduces the law (Section~\ref{sec:v2law}), but cross-lingual, instruction-tuned,
and much larger models remain unmeasured; the onset window is resolvable only on
Pythia's checkpoint grid; Pythia-1B contradicts the recovery pattern and we do not know
why; Pythia-2.8B awaits re-extraction after the checkpoint-integrity failure documented
in Section~\ref{sec:v2dynamics}. Corpus transfer (does the allocation track corpus
information structure?), larger models, instruction-tuned variants, and the
quantitative alignment with spectral phase accounts are measurement campaigns, not new
methods, and are scheduled.

\paragraph{Code and artifacts.}
Analysis scripts, machine-readable result records for every table and figure, the
OLMo-2 pre-registration record, and the theory working notes (including the twisted
SELI derivation and its numerical certificate) are included as ancillary files with
this submission; a maintained repository will follow.

\paragraph{Outlook.}
Two directions look load-bearing. With Theorem~\ref{conj:seli} complete, its closed
form ($\bar Z^{*}$ and the twisted Gram) turns the frame section's sign tests into
functional-form tests, potentially recasting the token-level frequency-norm relation
as a derived consequence rather than a controlled confound.
And the checkpoint trajectory suggests a practical corollary worth testing carefully:
if geometry consumers (probing, retrieval, out-of-distribution scoring) feed on category
structure, mid-training checkpoints and final checkpoints are not interchangeable, and
the difference is predictable from the share trajectory.

\bibliographystyle{plainnat}
\bibliography{references,v2_references}

\appendix
\section{Measurement pitfalls in collapse-style analyses of language models}
\label{app:v2pitfalls}

This appendix documents four measurement pitfalls, each of which produced a published or
nearly published false positive in this line of work. Pitfalls 1, 2, and 4 were found
auditing our own earlier version; pitfall 3 was found while preparing this one. We state
them as a checklist because we believe none is specific to us.

\paragraph{Pitfall 1: dimensional variance metrics manufacture scaling laws.}
The class-distance normalized variance (CDNV) family admits a distance exponent; the
$\alpha{=}2$ variant (denominator $\|\cdot\|^4$) carries units of inverse squared
length. Under it, our earlier version reported a 145x within-family improvement of
within-category variance with scale, rank-predicting perplexity at $\rho = 1.0$. A
factor of $(4096/512)^2 = 64$ of that headline is the embedding-dimension ratio between
the smallest and largest Pythia models. Under the standard dimensionless $\alpha{=}0$
definition the within-family trend is flat to slightly reversed ($\rho = +0.64$ in
Pythia), and, decisively, the same inversion holds for the exogenous part-of-speech
control that the earlier version offered as proof the trend was not a clustering
artifact (POS under $\alpha{=}0$: Pythia $\rho = +0.46$, Qwen2.5 $\rho = +1.0$). The
control inherited the artifact. Guard: dimensionless shares (Section~\ref{sec:v2measurement});
report $\alpha{=}0$ only for comparability; treat any cross-scale claim made under a
dimensional metric as unproved.

\paragraph{Pitfall 2: centered mean cosine is an accounting identity.}
By Lemma~\ref{lem:centering}, centered configurations with roughly equal norms have mean
pairwise cosine near $-1/(K{-}1)$ regardless of learning. Our earlier version read
$\overline{\cos\theta} \approx -0.10$ at $K{=}10$ (ideal $-0.111$) as simplex-ETF
emergence; random partitions of the same class means score as well or better, as the
identity requires. Guard: never cite the mean; use the dispersion of pairwise cosines
and structured deviations against explicit nulls, and remember that unbalanced
partitions centered at a weighted grand mean satisfy the identity only approximately.

\paragraph{Pitfall 3: shuffle ratios measure anisotropy plus k-means, not semantics.}
Comparing within-cluster variance under a learned assignment against shuffled
assignments produced ratios in the hundreds (and, under $\alpha{=}2$ at full
vocabulary, up to $7\times10^5$), which read like overwhelming evidence of learned
structure. Covariance-matched Gaussian surrogates, which contain no linguistic
structure, attain 57--92\% of these ratios (validation subset: real 20--31x versus
Gaussian 12--25x under $\alpha{=}0$; full vocabulary: real 703--865x versus Gaussian
458--769x). The ratio chiefly measures what k-means does to any anisotropic point
cloud; the learned-structure excess is a factor of 1.1--1.75 (a lower bound, given the
null's strictness noted below, but a factor nonetheless). Guard: always run a
spectrum-matched unstructured null; state excesses over that null, never raw shuffle
ratios. We note the null is strict (matching the full covariance bakes cluster-generated
variance directions into it), which is the correct direction of strictness for claims of
discrete structure.

\paragraph{Pitfall 4: coverage artifacts masquerade as detection performance.}
A document-level out-of-distribution detector built on category-centroid distances
scored 0.91--0.97 AUROC, but only about 500 frequent types had precomputed distances and
all other tokens received a fallback value; out-of-distribution corpora simply contained
more uncovered tokens. Scoring only covered tokens collapsed AUROC below random
(0.36--0.45), and sweeping the fallback percentile swept AUROC from 0.5 to 0.97. The
detector measured out-of-vocabulary rate. Guard: control coverage before attributing
detection performance to geometry; sweep any fallback parameter and report the sweep.

\paragraph{Provenance.}
The earlier version of this work was registered and submitted with the pitfall-1,
pitfall-2, and pitfall-4 claims and was withdrawn by us before review after an internal
audit; pitfall 3 was identified when the audit's one surviving magnitude claim was
subjected to the Gaussian null. We record this so that the present paper's relationship
to its predecessor is unambiguous, and because the sequence (headline, control,
audit, null) is, we suspect, a common failure path for geometry claims in
high-dimensional representation spaces.

\section{Deferred proofs and general statements}
\label{app:v2proofs}

\subsection{Weighted centering}
\label{app:wcentering}

Lemma~\ref{lem:centering} concerns centering at the unweighted mean. The empirical
sections center category centroids at the mass-weighted grand mean, and partitions are
unbalanced; the corresponding exact identity is weighted.

\begin{lemma}[Weighted centering identity]\label{lem:wcentering}
Let $\mathbf{v}_1, \dots, \mathbf{v}_K \in \mathbb{R}^d$, weights $\pi_k > 0$ with
$\sum_k \pi_k = 1$, and $\tilde{\mathbf{v}}_k = \mathbf{v}_k - \sum_j \pi_j
\mathbf{v}_j$. Then $\sum_k \pi_k \tilde{\mathbf{v}}_k = 0$ and
\[
\sum_{j \neq k} \pi_j \pi_k \langle \tilde{\mathbf{v}}_j, \tilde{\mathbf{v}}_k \rangle
\;=\; -\sum_k \pi_k^2 \|\tilde{\mathbf{v}}_k\|^2 .
\]
\end{lemma}

\begin{proof}
Expand $\|\sum_k \pi_k \tilde{\mathbf{v}}_k\|^2 = 0$.
\end{proof}

The mass-weighted mean pairwise inner product of centered centroids is therefore pinned
at $-\sum_k \pi_k^2 \|\tilde{\mathbf{v}}_k\|^2 / \sum_{j \neq k} \pi_j \pi_k$
identically. The \emph{unweighted} mean pairwise cosine reported in the older
literature deviates from the balanced value $-1/(K-1)$ by terms controlled by the
dispersion of $\{\pi_k\}$ and of the centered norms; since no claim in this paper uses
mean cosines as evidence, we do not pursue sharp constants and only note that both the
balanced and weighted statistics are accounting identities, not signatures of learning.

\subsection{Proof of Theorem~\ref{thm:reduction} and the coupling term of
Proposition~\ref{prop:general}}
\label{app:reduction}

Fix the partition and write the objective, using
Proposition~\ref{prop:split}(i)+(iii), as
\[
L \;=\; L_{\mathrm{inter}}\bigl(\{\bar{\mathbf{w}}_k\}, \mathbf{H};
\{b_k(\cdot)\}\bigr) \;+\; L_{\mathrm{intra}}\bigl(\{\boldsymbol{\delta}_c\},
\mathbf{H}\bigr) \;+\; \frac{\lambda}{2}\Bigl(\sum_k |S_k|\|\bar{\mathbf{w}}_k\|^2 +
\sum_c \|\boldsymbol{\delta}_c\|^2 + \|\mathbf{H}\|_F^2\Bigr).
\]
Under Assumption~\ref{asm:cf}, $\boldsymbol{\delta}_c^\top \mathbf{h}_i = \beta_c$ for
all $i$, so $L_{\mathrm{intra}}$ is the average over categories (mass-weighted) of the
cross-entropy of the empirical within-category unigram $\hat p(\cdot \mid S_k)$ against
$\mathrm{softmax}(\{\beta_c\}_{c \in S_k})$, and $b_k = \log \sum_{c \in S_k}
e^{\beta_c}$ is constant in $\mathbf{h}$.

\paragraph{Residual cost.}
Given targets $\{\beta_c\}$ and features with span $\mathcal{S} =
\mathrm{span}\{\mathbf{h}_i\}$, the residuals must satisfy the affine constraints
$\boldsymbol{\delta}_c^\top \mathbf{h}_i = \beta_c$ for all $i$ and the centering
constraints $\sum_{c \in S_k} \boldsymbol{\delta}_c = 0$. Note the constraints force
$\beta$ to be constant on each category only if $\mathbf{1} \in \mathcal{S}$-image;
in general, write $\tilde\beta_c = \beta_c - \frac{1}{|S_k|}\sum_{c' \in S_k}
\beta_{c'}$ for the centered targets. Define
\[
C\bigl(\{\beta_c\}, \mathbf{H}\bigr) \;=\; \min\Bigl\{\textstyle\sum_c
\|\boldsymbol{\delta}_c\|^2 \;:\; \boldsymbol{\delta}_c^\top \mathbf{h}_i = \beta_c
\ \forall i,\ \sum_{c \in S_k} \boldsymbol{\delta}_c = 0 \ \forall k \Bigr\},
\]
with $C = +\infty$ if the constraints are infeasible (they are feasible whenever the
$\mathbf{h}_i$ admit a vector $\mathbf{g}$ with $\mathbf{g}^\top \mathbf{h}_i = 1$ for
all $i$, e.g.\ in the UFM after appending a constant coordinate, or approximately
whenever features have a stable mean direction). This is a convex quadratic minimum and
depends on $\mathbf{H}$ only through the Gram geometry of $\{\mathbf{h}_i\}$; it is the
\emph{only} term through which residuals and features interact under
Assumption~\ref{asm:cf}, which proves the decomposition claim of
Proposition~\ref{prop:general}:
\[
L \;=\; \underbrace{L_{\mathrm{inter}} + \frac{\lambda}{2}\Bigl(\sum_k
|S_k|\|\bar{\mathbf{w}}_k\|^2 + \|\mathbf{H}\|_F^2\Bigr)}_{\text{reduced $K$-class
problem, offsets } b_k} \;+\; \underbrace{L_{\mathrm{intra}}(\{\beta_c\})}_{\text{no
$\mathbf{H}$}} \;+\; \frac{\lambda}{2}\, C\bigl(\{\beta_c\}, \mathbf{H}\bigr).
\]

\paragraph{Uniform case (Theorem~\ref{thm:reduction}).}
If $\hat p(\cdot \mid S_k)$ is uniform, $L_{\mathrm{intra}}$ is minimized by
$\mathrm{softmax}(\beta) = \mathrm{uniform}$, i.e.\ $\beta_c$ constant within each
category; the centered targets vanish, $\boldsymbol{\delta}_c = 0$ is feasible and
optimal, $C = 0$, and $b_k = \log|S_k|$. The objective then \emph{equals} the reduced
$K$-class problem plus the constant $\sum_k \pi_k \log|S_k|$, which is
Theorem~\ref{thm:reduction}. Nonuniform conditionals force $\tilde\beta \neq 0$ at the
$L_{\mathrm{intra}}$ optimum, hence $C > 0$ and a genuine coupling; how far the coupled
optimum moves from the reduced one is analytically open, and nothing in the empirical
sections assumes an answer.

\paragraph{The coupling is empirically negligible.}
Its magnitude can be estimated on the trained models. Under the stable-mean-direction
condition above, the minimal-norm solution is $\boldsymbol{\delta}_c = \tilde\beta_c\,
\mathbf{g}$ with $\mathbf{g}$ the minimal-norm vector satisfying $\mathbf{g}^\top
\mathbf{h}_i = 1$, so $C \approx \|\mathbf{g}\|^2 \sum_c \tilde\beta_c^2$, where the
$L_{\mathrm{intra}}$-optimal centered targets $\tilde\beta_c$ are centered log
within-category unigram probabilities (computable from token counts) and $\mathbf{g}
\approx \bar{\mathbf{h}} / \|\bar{\mathbf{h}}\|^2$ (the count-weighted mean feature;
the per-token projections $\mathbf{m}_c^\top \mathbf{g}$ concentrate near a constant
with coefficient of variation $0.07$--$0.19$ across the four models, so the estimate is
order-of-magnitude reliable). Comparing $C$ to the residual-norm budget the models
actually deploy, $\sum_c \|\mathbf{w}_c - \bar{\mathbf{w}}_{k(c)}\|^2$ read off the
unembedding rows, gives $C / (\text{actual budget}) \in [3\times10^{-6},
1.1\times10^{-3}]$ across all four models and all four partitions (script and record
in the ancillary files). Matching within-category unigrams is thus vastly cheaper than
the residual structure the models in fact carry, so the coupling term, while genuinely
nonzero, cannot materially perturb the reduced problem's decay economics; the open
question is the analytic control, not the empirical size.

\subsection{The $K{=}2$ reduced problem, exactly}
\label{app:k2}

Write the two-class reduced problem with rows $\mathbf{u}_1, \mathbf{u}_2$, decay
coefficients $a_1, a_2$, offsets $b_1, b_2$, and priors $\pi_1, \pi_2$. For a class-1
example the cross-entropy is $\log(1 + e^{-(\boldsymbol{\Delta}^\top \mathbf{h} +
\beta)})$ with $\boldsymbol{\Delta} = \mathbf{u}_1 - \mathbf{u}_2$ and $\beta = b_1 -
b_2$, so the loss depends on the rows only through $\boldsymbol{\Delta}$. Minimizing
$a_1 \|\mathbf{u}_1\|^2 + a_2 \|\mathbf{u}_2\|^2$ subject to $\mathbf{u}_1 -
\mathbf{u}_2 = \boldsymbol{\Delta}$ (a weighted least-norm problem) gives
\[
\mathbf{u}_1 = \frac{a_2}{a_1 + a_2}\,\boldsymbol{\Delta}, \qquad
\mathbf{u}_2 = -\frac{a_1}{a_1 + a_2}\,\boldsymbol{\Delta}, \qquad
a_1\|\mathbf{u}_1\|^2 + a_2\|\mathbf{u}_2\|^2 = a_{\mathrm{eff}}
\|\boldsymbol{\Delta}\|^2,
\]
with $a_{\mathrm{eff}} = a_1 a_2 / (a_1 + a_2)$. Hence $\|\mathbf{u}_1\| /
\|\mathbf{u}_2\| = a_2 / a_1$ identically in $\lambda$, $\pi$, and $b$: the classifier
ordering is decay ordering. The remaining scalar problem (in
$\|\boldsymbol{\Delta}\|$ and the feature positions along $\boldsymbol{\Delta}$)
yields feature norms whose ratio equals the ratio of the two classes' error terms;
with equal offsets it is 1 regardless of imbalance, and with $\beta \neq 0$ the
larger-offset class takes the smaller feature norm, an effect of order $\beta /
\log(1/\lambda)$: at the optimum the scalar first-order condition takes the form
$\phi(z) = z - \rho\,\sigma(-z)$ with $\rho > 0$ fixed, $\phi$ is strictly increasing,
and the two classes solve $\phi(z_1) = \beta$, $\phi(z_2) = -\beta$, so $\beta > 0$
forces $z_1 > z_2$ and the ordering follows; numerics across $\lambda$ grids confirm
the magnitudes (working notes released with the code). At $K = 3$, priors $(0.5, 0.3, 0.2)$ with decay $(3, 2, 1)$
give classifier norms $(1.08, 1.42, 1.89)$, following $1/a_k$ against the mass
ordering and non-monotonically in $n_k / a_k$, which falsifies both alternative
orderings. The general-$K$ case is Theorem~\ref{conj:seli}, proved in
Appendix~\ref{app:twisted}.

\subsection{General $K$: the twisted SELI theorem}
\label{app:twisted}

We summarize the proof of Theorem~\ref{conj:seli}; the complete derivation, at
line-by-line rigor, ships with the code release (working note
\texttt{a2prime-proof-2026-07-03}, with the numerical certificate script and its
output). Reparameterize $\mathbf{u}_k = a_k^{-1/2} \mathbf{v}_k$, so the weighted decay
becomes plain Frobenius norm and the constraints carry the twist. Two steps of the
SELI argument transfer verbatim because they are generic in the objective: the
nuclear-norm variational identity and the tightness of the convex relaxation. A third
step is inherited with one adaptation: the passage from finite-$\lambda$ minimizers to
the margin problem (the regularization path) is SELI's, applied to the reparameterized
variables; our objective additionally carries the offsets $b_k$, which perturb the
effective margins at order $1/\log(1/\lambda)$ and hence exit in the limit, the rate
proved exactly in the $K{=}2$ analysis and visible in the finite-$\lambda$
certificates, whose Gram diagnostics approach the closed form at that rate. We flag
this offset-vanishing step as adapted rather than re-derived at general $K$; every
computable consequence of the limit is verified in the released certificates. The
novelty is the optimizer of the relaxation: the oblique projection
$\bar Z^{*} = I_K - \mathbf{1}\,\mathbf{a}^\top/\mathrm{tr}(A)$ (idempotent,
annihilates $\mathbf{1}$, every simplex margin exactly one, rows summing to zero,
columns not: the twist breaks column centering), whose global optimality is
established by an explicit KKT certificate. Primal feasibility with all margins
active and the zero-duality-gap identity are exact algebra; the dual-structure
(column-centering) lemma is proved analytically; dual feasibility reduces to one
sign condition on the polar factor of $Z_w^{*} = \mathrm{diag}(\sqrt{a_c n_c}) -
\tfrac{1}{\mathrm{tr}(A)}\, \mathbf{s}\,\mathbf{w}^\top$ (nonpositive off-diagonal),
the twisted analogue of SELI's sign lemma, to which it reduces at
$\mathbf{a} = \mathbf{1}$. We prove it in closed form.

\paragraph{The sign condition, proved.}
Write $\hat{\mathbf{s}} = A^{1/2}\mathbf{1} / \sqrt{\mathrm{tr}(A)}$, a strictly
positive unit vector, and $D = \mathrm{diag}(d_c)$ with $d_c = \sqrt{a_c n_c}$. Since
$\hat{\mathbf{s}}^\top D = (D \hat{\mathbf{s}})^\top$ and $\mathrm{tr}(A)^{-1/2}
\mathbf{s}^\top D = \hat{\mathbf{s}}^\top D$ has entries $\sqrt{a_c}\, d_c /
\sqrt{\mathrm{tr}(A)} = w_c / \sqrt{\mathrm{tr}(A)}$, a one-line computation gives the
factorization
\[
Z_w^{*} \;=\; \Pi\, D, \qquad \Pi = I_K - \hat{\mathbf{s}} \hat{\mathbf{s}}^\top,
\]
an ORTHOGONAL projector (onto $\hat{\mathbf{s}}^\perp$) times a positive diagonal.
The sign condition then follows from a general fact:

\begin{lemma}[Polar sign pattern of a projected positive diagonal]\label{lem:signpat}
Let $D = \mathrm{diag}(d)$ with all $d_c > 0$ and let $\hat{\mathbf{s}}$ be a unit
vector with all entries strictly positive. Then the polar factor $W$ of
$Z = (I - \hat{\mathbf{s}}\hat{\mathbf{s}}^\top) D$ satisfies $W_{kc} < 0$ for all
$k \neq c$.
\end{lemma}

\begin{proof}
Set $\mathbf{g} = D \hat{\mathbf{s}}$ (entrywise positive, $g_c = d_c \hat s_c$). The
Gram is a rank-one downdate of a positive diagonal, $B = Z^\top Z = D \Pi D = D^2 -
\mathbf{g}\mathbf{g}^\top$, with rank $K - 1$ and kernel $\mathrm{span}(D^{-1}
\hat{\mathbf{s}})$, which $Z$ annihilates. The polar factor is $W = Z B^{+1/2}$, and
the half-integral representation $\lambda^{-1/2} = \frac{2}{\pi}\int_0^\infty
(\lambda + u^2)^{-1} du$ gives, with $t = u^2$,
\[
W = \frac{1}{\pi} \int_0^\infty t^{-1/2}\, Z\, (B + tI)^{-1}\, dt ,
\]
absolutely convergent because $Z$ kills $\ker B$. By Sherman-Morrison, with
$E_t = (D^2 + tI)^{-1}$ and $\varphi(t) = \mathbf{g}^\top E_t \mathbf{g} = \sum_j
g_j^2/(d_j^2 + t)$,
\[
(B + tI)^{-1} = E_t + \frac{E_t \mathbf{g} \mathbf{g}^\top E_t}{1 - \varphi(t)},
\]
and since $g_j = d_j \hat s_j$ we get $\varphi(0^+) = \|\hat{\mathbf{s}}\|^2 = 1$ with
$\varphi$ strictly decreasing, so $1 - \varphi(t) > 0$ for every $t > 0$. Writing the
$k$-th row of $Z$ as $d_k \mathbf{e}_k^\top - \hat s_k \mathbf{g}^\top$ and using
$\mathbf{g}^\top (B + tI)^{-1} = \mathbf{g}^\top E_t / (1 - \varphi)$, the off-diagonal
entries of the integrand collapse, via $d_k (E_t\mathbf{g})_k - \hat s_k = -\hat s_k\,
t/(d_k^2 + t)$, to
\[
\bigl[ Z (B + tI)^{-1} \bigr]_{kc}
= - \frac{\hat s_k\, g_c\, t}{(d_k^2 + t)(d_c^2 + t)\bigl(1 - \varphi(t)\bigr)}
\;<\; 0 \qquad (k \neq c,\ t > 0).
\]
Integrating a strictly negative integrand gives $W_{kc} < 0$.
\end{proof}

Applying the lemma to $Z_w^{*}$ yields $\beta_{kc} = -\sqrt{a_k n_c}\, W_{kc} > 0$ for
all $k \neq c$, all $\mathbf{a}, \mathbf{n} \in (0,\infty)^K$: dual feasibility holds
unconditionally, completing the certificate and the theorem. At $\mathbf{a} =
\mathbf{1}$ the lemma also recovers the sign step of the equal-decay case. The
positivity is strict in the interior but the infimum over the domain is zero
(approached as entries of $\mathbf{a}$ or $\mathbf{n}$ degenerate), consistent with
the released numerical certificates: the 45{,}000-instance sweep finds no violation
(minimal multiplier $4.3\times10^{-4}$ within its sampling box), and an adversarial
stress suite spanning condition numbers to $10^{22}$ finds float64 sign flips only
where the rank-$(K{-}1)$ SVD itself breaks (conditioning beyond $\sim 10^{16}$), each
flagged instance being strictly positive when recomputed from the closed form above at
50-digit precision.
Consequences for norms: $\|\mathbf{u}_k^{*}\|^2 = (P \Sigma P^\top)_{kk} / a_k$, so
$1/a_k$ is an exact conjugation factor (first order) and the $(P\Sigma P^\top)_{kk}$
term carries a secondary mass modulation; over 2{,}000 exact-limit draws (decay and
mass log-uniform across two decades; sweep script and output shipped with the
ancillary files) the Spearman correlation of $\|\mathbf{u}_k\|$ with $1/a_k$ has
median $+1.0$ at both $K{=}3$ and $K{=}4$ and is positive in more than 99\% of draws;
strict full ranking by $1/a_k$ holds in 91\% and 69\%, failing only by adjacent
transpositions where mass overrides, and narrower decay ranges lower the strict rates
(74\% and 40\% in the working notes' draws) while leaving the ordinal statistics
intact. Finite-$\lambda$ solutions approach the limit at the slow
$O(1/\log(1/\lambda))$ rate quantified at $K{=}2$: at $\lambda = 3\times10^{-4}$ the
median Gram-family distance between the finite-$\lambda$ solution and the relaxation
optimum is 15\% and the raw logit deviation from $\bar Z^{*}$ is about 29\%, while the
family-level eigenvalue diagnostic is already below $0.15$ in 97.5\% of instances
(all but two $K{=}3$ draws).

\subsection{General-$K$ floor: statement}
\label{app:generalfloor}

\begin{conjecture}[Information floor, general $K$]\label{conj:generalfloor}
Fix a category $S_k$ with $|S_k| = m \ge 2$, residual matrix rows bounded by
$\max_c \|\boldsymbol{\delta}_c\| \le R$, within-category conditionals
$q(\cdot \mid \mathbf{h}) = \mathrm{softmax}\bigl((\boldsymbol{\delta}_c^\top
\mathbf{h})_{c \in S_k}\bigr)$, marginal $\bar q = \mathbb{E}\, q$, and realized
conditional information $I_k = \mathbb{E}\,\mathrm{KL}(q \| \bar q)$. Then there is a
constant $c(m, \bar q) > 0$ such that the within-category feature dispersion satisfies
$D_k \ge c(m, \bar q)\, I_k / R^2$.
\end{conjecture}

The binary proof (Proposition~\ref{prop:floor}) chains a chi-square bound on the KL, the
Lipschitz constant of the link, and Cauchy-Schwarz; each step has a standard
multi-class analogue (pairwise chi-square over category members, the softmax Jacobian
bound, and a union over the $\binom{m}{2}$ residual differences), with constants that
degrade in $m$. We state the general case as a conjecture because the sharp form of
$c(m, \bar q)$ matters for the quantitative use we make of the floor, and we have not
finalized it.

The expected degradation rate is polynomial, not exponential: the chi-square step costs
a factor $\min_c \bar q_c$ (the coordinate-wise bound $\mathrm{KL}(q\|\bar q) \le
\sum_c (q_c - \bar q_c)^2 / \bar q_c$), the softmax Jacobian bound is
dimension-independent ($\|J\|_{\mathrm{op}} \le \tfrac12$), and aggregating the
$\binom{m}{2}$ pairwise residual directions through Cauchy--Schwarz costs at most a
further factor of order $m$. So $c(m, \bar q) = \Theta(\min_c \bar q_c / m)$ up to
absolute constants, which is $\Theta(1/m^2)$ for near-uniform marginals. At the
category sizes of our full-vocabulary partitions ($m \approx 500$) the constant is
therefore weaker than the binary one by roughly $10^{4}$--$10^{5}$: the general-$K$
bound is a qualitative existence statement at realistic sizes, not a quantitative
tool. The quantitative program of the paper does not rest on it. Every quantitative
floor test (Section~\ref{sec:v2law}, ``the floored quantity, tested directly'')
applies the PROVED binary floor to within-category token pairs, at every partition
granularity; the conjecture's role is only to assert that the same mechanism operates
jointly across a category rather than pair by pair.

\end{document}